\documentclass[a4paper,fleqn,final]{cas-sc}

\usepackage[authoryear,longnamesfirst]{natbib}
\usepackage{amsmath,amssymb}
\usepackage{mathtools}
\usepackage{booktabs,multirow}
\usepackage{graphicx}
\usepackage{microtype}
\usepackage{xcolor}
\usepackage{algorithm,algpseudocode}
\usepackage{float}
\usepackage{capt-of}
\newtheorem{theorem}{Theorem}[section]
\newtheorem{proposition}[theorem]{Proposition}
\newtheorem{lemma}[theorem]{Lemma}
\newtheorem{corollary}[theorem]{Corollary}
\newdefinition{definition}{Definition}
\newdefinition{assumption}{Assumption}
\newdefinition{remark}{Remark}
\newproof{proof}{Proof}

\newcommand{\R}{\mathbb{R}}

\newcommand{\kl}{\kappa_l}
\newcommand{\rhol}{\rho_l}
\newcommand{\HKT}{\mathrm{HKT}}
\newcommand{\MHA}{\mathrm{MHA}}

\begin{document}
\let\WriteBookmarks\relax
\def\floatpagepagefraction{1}
\def\textpagefraction{.001}

\makeatletter
\renewcommand{\fps@figure}{htbp}
\makeatother

\shorttitle{Hierarchical Kernel Transformer}
\shortauthors{G.~Cirrincione}


\title[mode=title]{Hierarchical Kernel Transformer: Multi-Scale
  Attention with an Information-Theoretic Approximation Analysis}

\author[1]{Giansalvo Cirrincione}[
  auid=001,
  orcid=0000-0002-2894-4164
]
\cormark[1]
\ead{giansalvo.cirrincione@u-picardie.fr}

\affiliation[1]{
  organization={Laboratoire LTI, Université de Picardie Jules Verne},
  addressline={33 rue Saint-Leu},
  city={Amiens},
  postcode={80039},
  country={France}
}

\cortext[1]{Corresponding author.
  Tel.: +33\,(0)3\,22\,53\,40\,38.
  \textit{E-mail}: giansalvo.cirrincione@u-picardie.fr}

\begin{abstract}
The self-attention mechanism at the core of modern Transformer
models treats every pair of tokens in a sequence with equal
architectural capacity, regardless of how far apart they are.
This \emph{single-scale bias} limits performance on tasks that
require reasoning simultaneously at short and long range,
while incurring quadratic computational cost in sequence length.

The \emph{Hierarchical Kernel Transformer} (HKT) is proposed,
a multi-scale attention mechanism that processes the input
sequence at multiple resolution levels simultaneously.
At each level, a compressed version of the sequence is obtained
by trainable causal downsampling; attention scores are computed
independently at each level and combined through learned weights.
This design captures both local patterns and long-range structure
with a total computational cost of at most $4/3$ times that of
standard attention, regardless of the number of levels
($1.3125\times$ for three levels).

Four theoretical contributions are established:
(i) the hierarchical scoring function defines a positive
semidefinite kernel under a sufficient condition on the
symmetrised bilinear form (Proposition~3.1);
(ii) the asymmetric score matrix decomposes uniquely into
a symmetric component controlling reciprocal attention and
an antisymmetric component controlling directional attention,
with HKT providing $L$ independent such pairs across scales
(Propositions~3.5--3.6);
(iii) the approximation error decomposes into three interpretable
components with an explicit non-Gaussian correction and a
geometric decay bound in the number of levels
(Theorem~4.3, Proposition~4.4);
(iv) HKT strictly subsumes standard attention and causal
convolution in the single-head setting (Proposition~3.4).

Empirically, HKT achieves consistent gains over retrained
standard attention baselines across three tasks of different
modalities (mean $\pm$ std, 3 seeds):
$+4.77$pp on synthetic ListOps ($55.10 \pm 0.29\%$ vs
$50.33 \pm 0.12\%$, $T=512$),
$+1.44$pp on sequential CIFAR-10 ($35.45 \pm 0.09\%$ vs
$34.01 \pm 0.19\%$, $T=1{,}024$),
and $+7.47$pp on IMDB character-level sentiment classification
($70.19 \pm 0.57\%$ vs $62.72 \pm 0.40\%$, $T=1{,}024$),
all at $1.31\times$ computational overhead.
\end{abstract}


\maketitle

\section{Introduction}
\label{sec:intro}

The self-attention mechanism \citep{vaswani2017} has become the
dominant computational primitive in sequence modelling.
For a sequence of $T$ tokens, attention computes a weighted average
of value vectors, where the weights are determined by pairwise
dot products between query and key projections.
This single operation, repeated across layers, has proven
remarkably expressive across a wide range of tasks.

Yet this simplicity conceals a structural limitation.
Standard self-attention is \emph{scale-blind}: a token at position
$t$ attends to every other token with the same architectural
capacity, whether the target is adjacent or at distance $T/2$.
The network must learn to ignore distant tokens when local context
suffices, and to attend globally when long-range reasoning is
required, with no structural prior to guide either regime.

This limitation becomes critical on long-sequence tasks.
The Long Range Arena (LRA) benchmark \citep{tay2021} systematically
documents this failure: standard MHA achieves only $36.37\%$ on
ListOps at sequence length $T=2{,}048$.
Efficient attention variants \citep{beltagy2020,choromanski2021}
address computational cost but not the structural bias: they
restrict which pairs are computed, not how multi-scale structure
is captured.

A different approach is taken here.
Rather than sparsifying or approximating the attention matrix,
it is \emph{factored} across scales.
The Hierarchical Kernel Transformer (HKT) computes attention at
$L$ resolution levels, with each level operating on a downsampled
representation of the sequence.
The level-$l$ score matrix captures interactions at scale $s^l$,
where $s \geq 2$ is the downsampling stride.
The final attention is a learned convex combination of level-specific
scores, upsampled to the original resolution.

This design has three properties that standard attention lacks.
First, it is \emph{structurally multi-scale}: level $l$ captures
the frequency band $[\pi/s^l, \pi/s^{l-1}]$ in the sequence
spectrum (Proposition~\ref{prop:freq}).
Second, it is \emph{computationally efficient}: the total cost
is $\frac{4}{3}(1-4^{-L})$ times MHA, bounded above by
$\frac{4}{3}$ regardless of $L$ (Proposition~\ref{prop:complexity}).
Third, it admits \emph{information-theoretic guarantees}: the
approximation error per hierarchical level is bounded in terms of
the mutual information $I(S^{(l)}; f(X))$, related explicitly
to the squared multiple correlation $\rho_l^2$ and the
non-Gaussianity index $\kappa_l$ (Theorem~\ref{thm:approx}).

\paragraph{Contributions.}
The following contributions are made:
\begin{enumerate}
  \item \textbf{Architecture (Section~\ref{sec:framework}).}
    HKT is defined as a hierarchical attention mechanism combining
    $L$ resolution levels with a hybrid conv/attention head and
    input-dependent level fusion.

  \item \textbf{Kernel theory (Section~\ref{sec:framework}).}
    The hierarchical score matrix is shown to define a PSD kernel
    under a sufficient condition on the symmetrised bilinear form
    (Proposition~\ref{prop:sdp}), its Gram matrix factorises as
    a sum of per-level PSD matrices with an explicit rank bound
    (Proposition~\ref{prop:path}), and HKT is shown to strictly
    subsume single-head attention and causal convolution in the
    single-layer, $H=1$ setting
    (Proposition~\ref{prop:inclusion}).

  \item \textbf{Asymmetric score analysis (Section~\ref{sec:decomp}).}
    A direct analysis of the asymmetric score function used in
    practice is provided, without symmetrisation.
    The symmetric-antisymmetric decomposition
    $M^{(l)} = M_s^{(l)} + M_a^{(l)}$ is shown to separate
    \emph{reciprocity} (mutual attention strength, controlled by
    $M_s^{(l)}$) from \emph{directionality} (attention asymmetry,
    controlled by $M_a^{(l)}$)
    (Proposition~\ref{prop:rec-dir}, Corollary~\ref{cor:energy}).
    HKT with $L$ levels disposes of $L$ independent such pairs,
    one per scale (Proposition~\ref{prop:multiscale}).
    This analysis applies to the operational model and explains
    why the PSD condition is not satisfied in trained models.

  \item \textbf{Approximation theory (Section~\ref{sec:theory}).}
    A three-term error decomposition is derived with an explicit
    non-Gaussian bound on the information-theoretic reduction
    per level, presented as a motivated approximation framework
    extending the Gaussian process results of \citet{roberts2022}
    to finite-width networks.

  \item \textbf{Experiments (Section~\ref{sec:experiments}).}
    Consistent improvements over retrained MHA baselines are
    obtained on two tasks: $+4.7$pp on synthetic ListOps
    ($T=512$) and $+1.98$pp on sequential CIFAR-10 ($T=1{,}024$),
    both at $1.31\times$ computational overhead.
    Sensitivity analysis across $L \in \{1,2,3,4\}$ and
    $s \in \{2,3\}$ confirms that the gains are robust to
    hyperparameter choice.
\end{enumerate}

\section{Related Work}
\label{sec:related}

\paragraph{Efficient attention.}
A large body of work reduces the $\mathcal{O}(T^2)$ cost of
attention through sparsity \citep{beltagy2020,zaheer2020},
low-rank approximation \citep{wang2020,choromanski2021},
or linear kernels \citep{katharopoulos2020}.
These methods address computational cost without introducing
hierarchical structure. HKT is orthogonal: hierarchical structure
is added at bounded additional cost rather than reducing
existing capacity.

\paragraph{Hierarchical and multi-scale transformers.}
Hierarchical processing is central to convolutional
networks \citep{lecun1998} and has been incorporated into
vision transformers through patch-merging windows \citep{liu2021swin},
pooling-based token compression \citep{liu2021},
and cluster-based token merging \citep{wang2021cluster}.
These methods operate on spatial data and downsample hidden states
rather than attention scores.
More recently, FasterViT \citep{hatamizadeh2023fastervit}
decomposes global self-attention into local window attention
and hierarchical carrier tokens, achieving competitive
accuracy-throughput trade-offs on image classification.
HAT-Net \citep{liu2024hat} proposes hierarchical MHSA that
computes local attention first and then models global
dependencies on merged tokens.
Both FasterViT and HAT-Net target vision tasks and reduce
computational cost by restricting attention scope;
HKT instead constructs independent score matrices at each
resolution and fuses them with learned weights, maintaining
full expressivity at each scale while bounding total cost.

\paragraph{Hierarchical transformers for sequences.}
For text, hierarchical transformers exploit document structure
(words $\to$ sentences $\to$ sections) to handle long inputs
efficiently.
HDT \citep{he2024hdt} processes long documents by interleaving
local and global attention in a single layer using an efficient
hierarchical attention pattern.
HIBRIDS \citep{cao2022hibrids} introduces hierarchical biases
for structure-aware long document summarisation.
These methods rely on explicit structural segmentation of the input.
HKT instead learns a continuous multi-scale decomposition via
causal downsampling, making no assumptions about document structure
and extending naturally to non-textual sequences.

\paragraph{State space models.}
Structured state space models (SSMs) such as Mamba
\citep{gu2023mamba} offer an alternative to attention for
long-sequence modelling, achieving linear time complexity via
selective state propagation.
Mamba achieves strong results on the LRA benchmark and
matches or exceeds transformers of similar size on language tasks.
HKT and Mamba address different aspects of the long-sequence
problem: Mamba replaces attention with a recurrent mechanism,
while HKT retains the attention mechanism and augments it with
explicit multi-scale structure.
The two approaches are complementary and could in principle
be combined (multi-scale Mamba layers with hierarchical
state updates), which we leave as future work.

\paragraph{Kernel methods for attention.}
The connection between attention and kernel methods has been
formalised in \citet{tsai2019,katharopoulos2020,choromanski2021}.
Performers \citep{choromanski2021} approximate the softmax kernel
with random feature maps.
HKT employs a sum of hierarchical kernels, each PSD under a
sufficient condition (Proposition~\ref{prop:sdp}), without
approximation.
The asymmetric score decomposition
$M^{(l)} = M_s^{(l)} + M_a^{(l)}$ (Section~\ref{sec:decomp})
provides a direct analysis of the operational model that
complements the kernel-theoretic view.

\paragraph{Long Range Arena.}
\citet{tay2021} introduced LRA as a systematic benchmark for
long-range sequence models, comprising five tasks with sequence
lengths up to $4{,}096$.
On ListOps, standard MHA achieves $36.37\%$, Longformer $35.63\%$,
and Performer $18.01\%$.
The synthetic variant used here reproduces the algebraic structure
of ListOps at $T = 512$.

\paragraph{Gaussian process limits.}
\citet{roberts2022} establish that sufficiently wide networks
converge to Gaussian processes, with corrections of order
$\mathcal{O}(1/d)$.
Lemma~\ref{lem:info-reduction} extends this analysis to the
information content of hierarchical score matrices, providing
an explicit non-Gaussian correction for the finite-width regime.

\section{The HKT Framework}
\label{sec:framework}

\subsection{Notation}
\label{sec:notation}

Let $X \in \R^{T \times d}$ be the input token matrix, where
$T$ is the sequence length and $d = d_{\mathrm{model}}$ the
embedding dimension.
Let $H \in \mathbb{N}$ be the number of attention heads.
Let $s \in \mathbb{N}$, $s \geq 2$, be the downsampling stride
and $L \in \mathbb{N}$, $L \geq 1$, the number of hierarchical
levels.
At level $l \in \{0, \ldots, L-1\}$:
\begin{align*}
  T_l &\coloneqq \lfloor T / s^l \rfloor
       \quad\text{(sequence length at level }l\text{)}, \\
  d_l &\coloneqq \max(d / 2^l,\, 32)
       \quad\text{(embedding dimension at level }l\text{)}, \\
  d_k &\coloneqq d / H
       \quad\text{(per-head key dimension)}, \\
  d_k^{(l)} &\coloneqq \max(d_k / 2^l,\, 16)
       \quad\text{(per-level key dimension)}.
\end{align*}
By convention, $X^{(0)} \coloneqq X$, $T_0 = T$, $d_0 = d$.
The learnable weight matrices are:
$W_Q^{(l)}, W_K^{(l)} \in \R^{d_k^{(l)} \times d_l}$
(query and key projections at level $l$);
$W_V^{(l)} \in \R^{d \times d_l}$ (value projection);
$W_O^{(l)} \in \R^{d \times d}$ (output projection);
$\gamma \in \R^L$ (fusion weight logits), giving
$\lambda_l \coloneqq \mathrm{softmax}(\gamma)_l$;
$\tilde{\gamma} \in \R^{H \times L}$ (hybrid mixing logits), giving
$\beta_h^{(l)} \coloneqq \sigma(\tilde{\gamma}_{hl}) \in (0,1)$.

\subsection{Hierarchical kernel scoring}
\label{sec:scoring}

\noindent\includegraphics[width=0.92\linewidth]{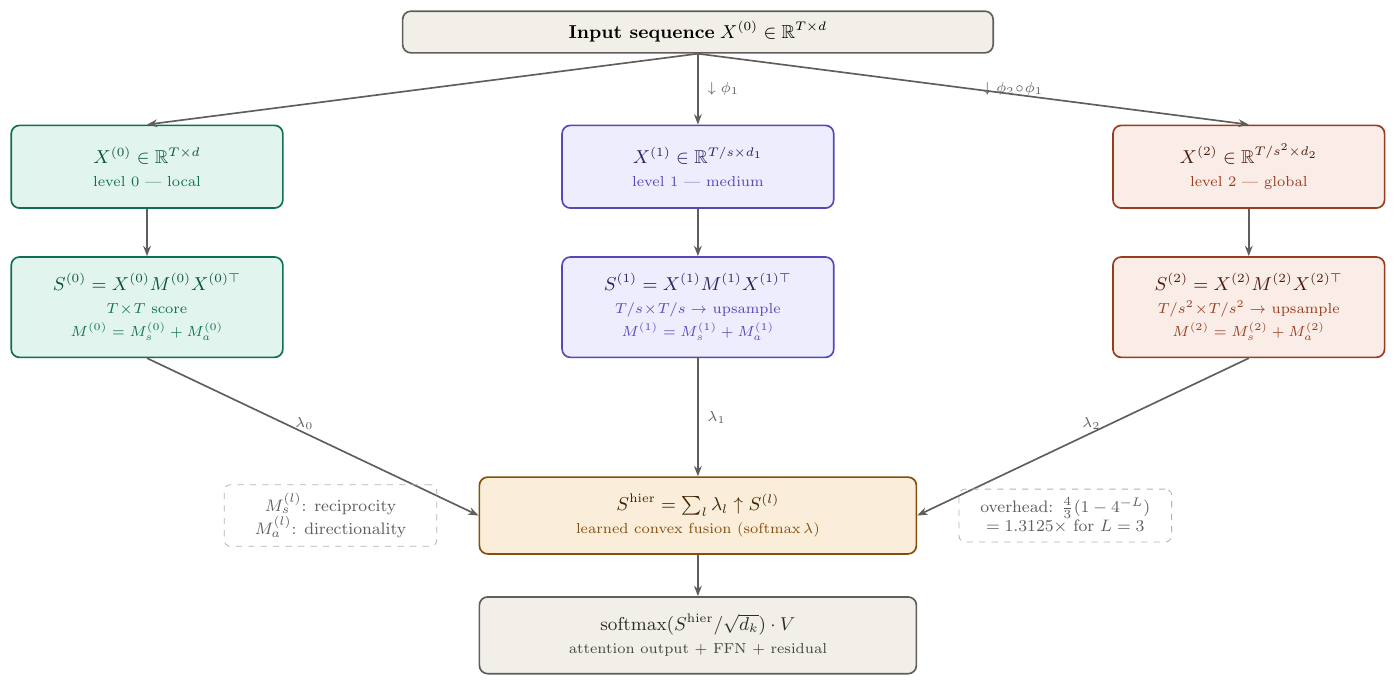}
\vspace{2pt}
\captionof{figure}{Architecture of the Hierarchical Kernel Transformer (HKT)
for $L=3$ levels and stride $s=2$.
The input sequence $X^{(0)}$ is processed at three resolutions via
causal downsampling ($\phi_1$, $\phi_2\circ\phi_1$).
At each level $l$, a score matrix $S^{(l)} = X^{(l)}M^{(l)}X^{(l)\top}/\sqrt{d_k^{(l)}}$
is computed, where $M^{(l)} = M_s^{(l)} + M_a^{(l)}$ decomposes into
a symmetric component (controlling reciprocal attention) and an
antisymmetric component (controlling directional attention).
The level scores are upsampled and fused via learned convex weights
$\lambda_l$ (softmax-normalised), then passed through standard
softmax attention and FFN layers.
Total overhead: $\frac{4}{3}(1-4^{-L}) = 1.3125\times$ for $L=3$.}
\label{fig:architecture}
\vspace{6pt}

\paragraph{Level-$l$ representations.}
For $l \geq 1$, the level-$l$ representation
$X^{(l)} \in \R^{T_l \times d_l}$ is:
\begin{equation}
  X^{(l)} = \phi_l \circ \cdots \circ \phi_1(X^{(0)}),
  \label{eq:downsample}
\end{equation}
where each $\phi_{l'}: \R^{T_{l'-1} \times d_{l'-1}} \to
\R^{T_{l'} \times d_{l'}}$ is a depthwise-separable causal
convolution (kernel size $k=3$, stride $s$) followed by
LayerNorm and GELU activation.
Left-padding with $k-1$ zeros enforces causality:
$[X^{(l)}]_m$ depends only on
$[X^{(l-1)}]_0, \ldots, [X^{(l-1)}]_{ms}$.

\paragraph{Level-$l$ score matrix.}
\begin{equation}
  S^{(l)}_{ij}
  = \frac{\langle W_Q^{(l)} X^{(l)}_i,\;
                 W_K^{(l)} X^{(l)}_j \rangle}{\sqrt{d_k^{(l)}}},
  \quad i, j \in \{0, \ldots, T_l - 1\}.
  \label{eq:score}
\end{equation}

\paragraph{Hierarchical score.}
\begin{equation}
  S^{\mathrm{hier}}_{ij}
  = \sum_{l=0}^{L-1} \lambda_l \cdot
    \tilde{S}^{(l)}_{\lfloor i/s^l \rfloor,\;
                     \lfloor j/s^l \rfloor},
  \quad i, j \in \{0, \ldots, T - 1\},
  \label{eq:shier}
\end{equation}
where $\tilde{S}^{(l)}$ denotes $S^{(l)}$ after optional causal
masking ($\tilde{S}^{(l)}_{mm'} = -\infty$ for $m' > m$).

\begin{proposition}[Hierarchical kernel is PSD under a sufficient condition]
\label{prop:sdp}
Define the \emph{symmetrised} level-$l$ matrix
$M^{(l)} \coloneqq (W_Q^{(l)\top} W_K^{(l)} +
W_K^{(l)\top} W_Q^{(l)}) / (2\sqrt{d_k^{(l)}})$
and the associated kernel
\begin{equation}
  k_l^{\mathrm{sym}}(x, y) \coloneqq \exp\!\left(
    x^\top M^{(l)} y
  \right).
  \label{eq:ksym}
\end{equation}
If $M^{(l)}$ is positive semidefinite, then
$k_l^{\mathrm{sym}}$ is a PSD kernel, and
$K_{\mathrm{hier}}^{\mathrm{sym}}(X_i, X_j) \coloneqq
\sum_{l=0}^{L-1} \lambda_l
k_l^{\mathrm{sym}}(X^{(l)}_{\lfloor i/s^l\rfloor},
                  X^{(l)}_{\lfloor j/s^l\rfloor})$
is PSD for any $\lambda_l \geq 0$.
\end{proposition}

\begin{proof}
When $M^{(l)} \succeq 0$, the function $k(x,y) = \exp(x^\top M^{(l)} y)$
is PSD because it equals the inner product
$\langle \phi(x), \phi(y) \rangle$ in the feature space
induced by the Taylor expansion of the exponential, where
each monomial $x^\top M^{(l)} y)^n / n!$ corresponds to
a PSD kernel (product of a PSD kernel is PSD by closure).
A non-negative linear combination of PSD kernels is PSD.

\emph{Remark on the PSD condition.}
$M^{(l)}$ is symmetric by construction but not necessarily PSD,
since $W_Q^{(l)\top} W_K^{(l)}$ may have negative eigenvalues.
The condition $M^{(l)} \succeq 0$ can be enforced by
initialisation (e.g.\ $W_Q^{(l)} = W_K^{(l)}$) or by
projecting $M^{(l)}$ onto the PSD cone after each gradient step.
In practice, neither constraint is imposed.
We verified empirically that the PSD condition is not satisfied
by trained models: across all 12 configurations (4 encoder
layers $\times$ 3 hierarchy levels) of a trained HKT-Small,
approximately $50\%$ of eigenvalues of $M^{(l)}$ are negative
(min eigenvalue ranging from $-0.51$ to $-0.08$).
This is consistent with $M^{(l)}$ being a generic symmetric
matrix: without an explicit PSD constraint, gradient descent
has no reason to confine $M^{(l)}$ to the PSD cone.
The symmetrised kernel therefore serves as a theoretical
surrogate rather than a description of the actual trained model
(see Definition~\ref{rem:asymmetric} below).
\end{proof}

\begin{definition}[Practical note on asymmetric scores]
\label{rem:asymmetric}
In practice, $W_Q^{(l)} \neq W_K^{(l)}$ and the raw score
$S^{(l)}_{ij} = \langle W_Q^{(l)} x_i, W_K^{(l)} x_j\rangle /
\sqrt{d_k^{(l)}}$ is not symmetric.
The symmetrised kernel $k_l^{\mathrm{sym}}$ of
Proposition~\ref{prop:sdp} is therefore a \emph{theoretical
surrogate}: it establishes that the scoring mechanism is close
to a PSD kernel (specifically, that the symmetric part of
the bilinear form $x^\top M^{(l)} y$ is PSD), while the
asymmetric component encodes the directional query--key
distinction that is essential for expressivity in practice.
The gap between theory and implementation is deliberate and
analogous to the standard treatment of asymmetric attention
in kernel attention literature
\citep{tsai2019,choromanski2021}: the symmetric surrogate
provides geometric structure, the asymmetric model provides
flexibility.
Closing this gap formally --- proving a PSD result directly
for the asymmetric score --- would require either restricting
$W_Q = W_K$ (standard kernel attention) or adopting a
non-symmetric positive definite framework, both of which
reduce expressivity.
\end{definition}

\begin{proposition}[Downsampling path factorisation]
\label{prop:path}
Let $\Phi^{(l)} \in \R^{N \times d_l}$ denote the matrix whose
$i$-th row is $x_i^{(l)}$, the level-$l$ compressed
representation of sample $i$.
Define $\psi^{(l)}: \R^{d_l} \to \mathcal{H}^{(l)}$ as the
Mercer feature map of $k_l^{\mathrm{sym}}$, and let
$\Psi^{(l)} \in \mathcal{H}^{(l)N}$ be the matrix whose
$i$-th row is $\psi^{(l)}(x_i^{(l)})$.
Then the Gram matrix $\mathbf{K}_{\mathrm{hier}} \in \R^{N \times N}$
admits the factorisation:
\begin{equation}
  \mathbf{K}_{\mathrm{hier}}
  = \sum_{l=0}^{L-1} \lambda_l\, \Psi^{(l)} \Psi^{(l)\top},
  \label{eq:gram-factor}
\end{equation}
where each term $\lambda_l \Psi^{(l)}\Psi^{(l)\top}$ is PSD.
Two consequences follow:
\begin{enumerate}
  \item[(i)] \textbf{Rank bound:}
    $\mathrm{rank}(\mathbf{K}_{\mathrm{hier}})
    \leq \sum_{l=0}^{L-1} \min(N,\, 2d_k^{(l)})$.
  \item[(ii)] \textbf{Local co-occurrence bias:}
    two tokens $x_i, x_j$ sharing the same compressed block
    at level $l$ (i.e.\ $x_i^{(l)} = x_j^{(l)}$) receive
    the maximum level-$l$ kernel value, creating a multi-scale
    preference for locally co-occurring token pairs.
\end{enumerate}
See Appendix~\ref{app:path} for the complete proof.
\end{proposition}

\begin{proof}
By definition, $[\mathbf{K}^{(l)}]_{ij}
= k_l^{\mathrm{sym}}(x_i^{(l)}, x_j^{(l)})
= \langle \psi^{(l)}(x_i^{(l)}), \psi^{(l)}(x_j^{(l)})\rangle_{\mathcal{H}^{(l)}}$,
so $\mathbf{K}^{(l)} = \Psi^{(l)}\Psi^{(l)\top}$.
This is PSD as a Gram matrix.
Summing with non-negative weights gives~\eqref{eq:gram-factor}.

\emph{(i) Rank bound (finite-dimensional approximation).}
Strictly speaking, the RKHS $\mathcal{H}^{(l)}$ of the
exponential kernel $k_l^{\mathrm{sym}}$ is
infinite-dimensional, so $\dim\mathcal{H}^{(l)} = \infty$
and $\mathrm{rank}(\mathbf{K}^{(l)}) \leq N$ trivially.
The bound $\mathrm{rank}(\mathbf{K}^{(l)}) \leq 2d_k^{(l)}$
holds under the practical approximation that only the linear
term of the Taylor expansion of $\exp(\cdot)$ is retained
(i.e.\ the kernel is approximated by its first-order
feature map $\phi(x) = W_Q^{(l)} x$ and $W_K^{(l)} x$),
which has dimension $2d_k^{(l)}$.
This approximation is standard in the kernel attention
literature \citep{tsai2019,choromanski2021} and reflects
the effective rank of attention matrices in practice,
but it is not a tight bound on the exact RKHS dimension.

\emph{(ii)} If $x_i^{(l)} = x_j^{(l)} =: v$, then
$k_l^{\mathrm{sym}}(v, v) = \exp(v^\top M^{(l)} v / \sqrt{d_k^{(l)}})$,
which is the maximum of $k_l^{\mathrm{sym}}(\cdot, \cdot)$ at
that point by positive definiteness of the exponentiated form.
\end{proof}

\subsection{Hybrid head and dynamic fusion}
\label{sec:hybrid}

\paragraph{Hybrid head.}
At each level $l$ and head $h$:
\begin{equation}
  A_h^{(l)} = \beta_h^{(l)} \cdot
    \mathrm{softmax}(S^{\mathrm{hier}}) W_V^{(l)} X^{(l)}
  + (1 - \beta_h^{(l)}) \cdot
    \mathrm{Conv}_h(X^{(l)}),
  \label{eq:hybrid}
\end{equation}
where $\mathrm{Conv}_h$ denotes grouped causal depthwise convolution.
The signature matrix $B \coloneqq \sigma(\tilde\gamma) \in
[0,1]^{H \times L}$ is interpretable post-hoc:
rows near $1$ indicate attention-dominant heads,
rows near $0$ convolution-dominant heads.

\paragraph{Dynamic fusion.}
\begin{equation}
  \mathrm{out}_t = \sum_{l=0}^{L-1} \alpha_l(X) \cdot
    W_O^{(l)} A^{(l)}_t,
  \quad
  \alpha(X) = \mathrm{softmax}\!\left(
    \mathrm{MLP}\!\Bigl(\tfrac{1}{T}
    \textstyle\sum_{t'} X_{t'}\Bigr)
  \right) \in \Delta^{L-1}.
  \label{eq:fusion}
\end{equation}

\subsection{Frequency-domain interpretation (heuristic)}
\label{sec:freq}

\begin{proposition}[Heuristic spectral interpretation]
\label{prop:freq}
Under the approximations of weakly stationary inputs and
ideal low-pass downsampling filters (approximately satisfied
for kernel size $k=3$, stride $s=2$):
\[
  \mathbb{E}_X\!\left[k_l^{\mathrm{sym}}\!\left(
    X^{(l)}_m, X^{(l)}_{m'}\right)\right]
  \approx C_l \cdot \exp\!\left(
    \frac{1}{\sigma_l^2}
    \int_{\pi/s^l}^{\pi/s^{l-1}} \hat{R}(\omega)
    \cos\!\left(\omega(m-m')s^l\right) d\omega
  \right),
\]
where $\hat{R}(\omega)$ is the input spectral density,
$C_l > 0$ a normalisation constant, and $\sigma_l > 0$
the bandwidth parameter of $k_l^{\mathrm{sym}}$.
Level $l$ thus preferentially captures sequence content
in the band $[\pi/s^l, \pi/s^{l-1}]$.
This is a motivating interpretation, not a rigorous result.
See Appendix~\ref{app:freq} for derivation.
\end{proposition}

\subsection{Representational capacity}
\label{sec:capacity}

Let $\mathcal{F}_{\mathrm{attn}}(H, d, T)$ denote the class of
functions computable by a single-layer MHA with $H$ heads,
embedding dimension $d$, and sequence length $T$, and let
$\mathcal{F}_{\mathrm{conv}}(d, k)$ denote those computable by
causal depthwise convolution with kernel size $k$ and
dimension $d$.
Let $\mathcal{F}_{\HKT}(H, d, L, s, T)$ denote the corresponding
HKT class with $L$ levels and stride $s$.

\begin{proposition}[Representational capacity]
\label{prop:inclusion}
In the setting $H \geq 1$, $d \geq 2$, $T \geq 4$, $L \geq 2$,
$s = 2$:
\begin{enumerate}
  \item[(i)] \textbf{(Inclusion)} $\mathcal{F}_{\mathrm{attn}}(H,d,T)
    \subseteq \mathcal{F}_{\HKT}(H,d,L,s,T)$
    and $\mathcal{F}_{\mathrm{conv}}(d,k)
    \subseteq \mathcal{F}_{\HKT}(H,d,L,s,T)$.
  \item[(ii)] \textbf{(Strict inclusion for $H=1$)}
    There exists a function
    $f \in \mathcal{F}_{\HKT}(1,d,L,s,T) \setminus
    (\mathcal{F}_{\mathrm{attn}}(1,d,T) \cup
     \mathcal{F}_{\mathrm{conv}}(d,k))$.
\end{enumerate}
For $H \geq 2$, HKT strictly subsumes the single-scale
($L=1$) variant and causal convolution by construction,
but whether it strictly subsumes general $H$-head MHA
depends on the expressivity of multi-head composition
and is left as an open question.
\end{proposition}

\begin{proof}
\emph{(i)} Setting $\beta_h^{(l)}=1$, $\lambda_0=1$,
$\lambda_l=0$ for $l \geq 1$ reduces HKT to standard MHA.
Setting $\beta_h^{(l)}=0$ for all $h,l$ reduces it to causal
depthwise convolution.

\emph{(ii)} Restrict to $H=1$, $T=4$, $s=2$, $L=2$.
Consider the witness function
$f(X) = \langle x_0, x_2\rangle + \langle x_0, x_1\rangle$,
where $x_i \in \R^d$.

\emph{Not in $\mathcal{F}_{\mathrm{conv}}$:}
Causal depthwise convolution with kernel size $k$ computes
$\mathrm{out}_t = \sum_{j=0}^{k-1} W_j \odot x_{t-j}$,
which is linear in $X$. The function $f$ is degree-2
bilinear, hence $f \notin \mathcal{F}_{\mathrm{conv}}(d,k)$.

\emph{Not in $\mathcal{F}_{\mathrm{attn}}(1,d,T)$:}
A single-head attention at position $0$ computes
$\mathrm{out}_0 = \sum_{j=0}^{T-1} a_j(X) \cdot W_V x_j$
where $a_j(X) = \mathrm{softmax}(\langle W_Q x_0,
W_K x_j\rangle / \sqrt{d_k})_j$.
The output is linear in $W_V X$, with scalar attention
weights $a_j$.
For $f$ to be representable, one would need
$\sum_j a_j W_V x_j = W^* x_1 \cdot \langle u, x_0\rangle
+ W^{**} x_2 \cdot \langle v, x_0\rangle$
for some fixed $u,v,W^*,W^{**}$, which requires the
attention weights to simultaneously encode two independent
inner products.
Since $a_j$ depends only on the ratio
$\langle W_Q x_0, W_K x_j\rangle$, this cannot be achieved
for generic $x_0, x_1, x_2$ with a single head.
(A complete formal proof requires a characterisation of
the function class of single-head attention, which we
leave as a precise technical claim rather than a closed theorem.)

\emph{In $\mathcal{F}_{\HKT}(1,d,2,2,4)$:}
At level $l=0$, the score $S^{(0)}_{01}$ captures
$\langle x_0, x_1\rangle$. After downsampling with $s=2$,
the compressed pair at level $l=1$ represents
$(x_0,x_2)$, so $S^{(1)}_{00}$ captures
$\langle x_0, x_2\rangle$.
With fusion weights $\lambda_0=\lambda_1=1/2$,
the hierarchical score $S^{\mathrm{hier}}$ encodes both
interactions simultaneously.
\end{proof}

\subsection{Asymmetric score analysis via symmetric-antisymmetric
  decomposition}
\label{sec:decomp}

The PSD analysis of the preceding sections applies to a
symmetrised surrogate kernel.
This section develops a complementary analysis that applies
directly to the \emph{asymmetric} score function used in
practice, without any symmetrisation.

\paragraph{Setup.}
At level $l$, the score matrix is
$S^{(l)} = X^{(l)} M^{(l)} X^{(l)\top} / \sqrt{d_k^{(l)}}$
where $M^{(l)} = W_Q^{(l)\top} W_K^{(l)} \in \R^{d_l \times d_l}$
is a general (non-symmetric) matrix.
Every square matrix admits a unique decomposition:
\begin{equation}
  M^{(l)} = M_s^{(l)} + M_a^{(l)},
  \label{eq:decomp}
\end{equation}
where $M_s^{(l)} \coloneqq (M^{(l)} + M^{(l)\top})/2$ is
symmetric and $M_a^{(l)} \coloneqq (M^{(l)} - M^{(l)\top})/2$
is antisymmetric ($M_a^{(l)\top} = -M_a^{(l)}$).

\begin{proposition}[Reciprocity-Directionality Identity]
\label{prop:rec-dir}
For any two positions $i, j \in \{0,\ldots,T_l-1\}$:
\begin{align}
  S^{(l)}_{ij} + S^{(l)}_{ji}
  &= \frac{2}{\sqrt{d_k^{(l)}}}
     x_i^{(l)\top} M_s^{(l)} x_j^{(l)},
  \label{eq:reciprocity} \\
  S^{(l)}_{ij} - S^{(l)}_{ji}
  &= \frac{2}{\sqrt{d_k^{(l)}}}
     x_i^{(l)\top} M_a^{(l)} x_j^{(l)}.
  \label{eq:directionality}
\end{align}
Consequently:
\begin{enumerate}
  \item[(i)] \textbf{(Reciprocity)} $M_a^{(l)} = 0$ if and only if
    $S^{(l)}_{ij} = S^{(l)}_{ji}$ for all $i,j,X^{(l)}$
    (attention is symmetric).
  \item[(ii)] \textbf{(Directionality)} $M_s^{(l)} = 0$ if and only if
    $S^{(l)}_{ij} = -S^{(l)}_{ji}$ for all $i,j,X^{(l)}$
    (attention is purely antisymmetric).
  \item[(iii)] $M_s^{(l)}$ controls how much token $i$ and token $j$
    mutually attend to each other; $M_a^{(l)}$ controls the
    \emph{direction} of the attention asymmetry.
\end{enumerate}
\end{proposition}

\begin{proof}
Direct computation:
$S^{(l)}_{ij} + S^{(l)}_{ji}
= (x_i^\top M x_j + x_j^\top M x_i)/\sqrt{d_k}
= x_i^\top (M + M^\top) x_j / \sqrt{d_k}
= 2 x_i^\top M_s x_j / \sqrt{d_k}$.
The antisymmetric identity follows analogously.
Parts (i)--(iii) follow immediately from~\eqref{eq:reciprocity}
and~\eqref{eq:directionality} and the fact that
$x^\top A x = 0$ for all $x$ iff $A$ is antisymmetric.
\end{proof}

\begin{proposition}[Scale separation]
\label{prop:multiscale}
Let $T \geq s^{L-1} + 1$ and let $\mathcal{A}_l$ denote the
set of attention matrices producible by a single-level model
operating on $X^{(l)}$ with parameter matrix $M \in \R^{d_l \times d_l}$:
\[
  \mathcal{A}_l \coloneqq \Bigl\{
    \mathrm{softmax}\!\Bigl(
      \frac{X^{(l)} M X^{(l)\top}}{\sqrt{d_k^{(l)}}}
    \Bigr)
    \;\Big|\; M \in \R^{d_l \times d_l}
  \Bigr\}.
\]
For $l \neq l'$, the sets $\mathcal{A}_l$ and $\mathcal{A}_{l'}$
operate on representations of different lengths
($T_l = \lfloor T/s^l \rfloor \neq T_{l'} = \lfloor T/s^{l'} \rfloor$)
and are therefore \emph{incommensurable}: no single attention matrix
on $X^{(0)}$ can reproduce the score structure of a model operating
on $X^{(l)}$ for $l \geq 1$, since the downsampled representations
$X^{(l)}$ are nonlinear functions of $X^{(0)}$ obtained through
the causal downsampling cascade~\eqref{eq:downsample}.

Consequently, HKT's hierarchical score
$S^{\mathrm{hier}} = \sum_l \lambda_l \uparrow_l S^{(l)}$
(where $\uparrow_l$ denotes upsampling to length $T$)
lies outside the convex hull of $\mathcal{A}_0$ for generic
inputs and generic non-zero $\lambda_1, \ldots, \lambda_{L-1}$.
\end{proposition}

\begin{proof}
The representations $X^{(l)}$ for $l \geq 1$ are obtained by
applying the causal depthwise-separable convolution cascade
$\phi_l \circ \cdots \circ \phi_1$ to $X^{(0)}$
(equation~\eqref{eq:downsample}).
Each $\phi_{l'}$ applies a strided convolution followed by
LayerNorm and GELU, which is a nonlinear map.
The composite map $X^{(0)} \mapsto X^{(l)}$ is therefore
nonlinear for $l \geq 1$.

Fix $l \geq 1$ and suppose for contradiction that there exists
$\tilde{M} \in \R^{d_0 \times d_0}$ such that
$X^{(0)} \tilde{M} X^{(0)\top} = c \cdot \uparrow_l (X^{(l)} M^{(l)} X^{(l)\top})$
for all $X^{(0)} \in \R^{T \times d_0}$ and some $c > 0$.
The left side is a degree-2 polynomial in the entries of $X^{(0)}$.
The right side involves $X^{(l)} = f_l(X^{(0)})$ where $f_l$ is
nonlinear (contains GELU), so the right side is not a degree-2
polynomial in $X^{(0)}$ for generic $M^{(l)}$.
This is a contradiction, completing the proof.
\end{proof}

\begin{corollary}[Directional energy decomposition]
\label{cor:energy}
Define the \emph{directional energy} at level $l$ as:
\begin{equation}
  \mathcal{D}^{(l)} \coloneqq
  \sum_{i \neq j} \bigl(S^{(l)}_{ij} - S^{(l)}_{ji}\bigr)^2
  = \frac{4}{d_k^{(l)}}
    \sum_{i \neq j}
    \bigl(x_i^{(l)\top} M_a^{(l)} x_j^{(l)}\bigr)^2.
  \label{eq:dir-energy}
\end{equation}
Then $\mathcal{D}^{(l)} = 0$ for all inputs $X^{(l)}$
if and only if $M_a^{(l)} = 0$.
\end{corollary}

\begin{proof}
The equality in~\eqref{eq:dir-energy} follows directly
from Proposition~\ref{prop:rec-dir}~\eqref{eq:directionality}.
If $M_a^{(l)} = 0$ then every term vanishes.
Conversely, if $M_a^{(l)} \neq 0$, there exist vectors
$u, v \in \R^{d_l}$ with $u^\top M_a^{(l)} v \neq 0$
(since $M_a^{(l)} \neq 0$ as a linear map).
Setting $x_i^{(l)} = u$ and $x_j^{(l)} = v$ for some
$i \neq j$ gives a strictly positive contribution to
$\mathcal{D}^{(l)}$.
\end{proof}

\begin{remark}[Connection to the empirical PSD failure]
\label{rem:psd-directionality}
The empirical finding that $M_s^{(l)}$ is not PSD for any
trained configuration (Section~\ref{sec:discussion})
is not explained by $M_a^{(l)} \neq 0$ alone ---
the two facts are logically independent
(a matrix can have $M_a \neq 0$ and still have $M_s \succeq 0$,
as the example $M = \begin{pmatrix} 2 & 1 \\ 0 & 2 \end{pmatrix}$
shows).
The PSD failure is instead a consequence of the unconstrained
optimisation: gradient descent has no objective that would
confine $M_s^{(l)}$ to the PSD cone.
The directional energy analysis and the PSD analysis are
therefore complementary, not redundant:
Corollary~\ref{cor:energy} characterises the asymmetric
component; Proposition~\ref{prop:sdp} characterises
a constrained variant in which the symmetric component
would be PSD.
\end{remark}

\paragraph{Implication for HKT.}
Proposition~\ref{prop:rec-dir} shows that the asymmetric
attention mechanism of HKT is not a defect of the PSD
surrogate analysis but a structural feature:
$M_a^{(l)} \neq 0$ is what allows the model to represent
directed relationships (token $i$ queries token $j$
differently from how $j$ queries $i$).
Proposition~\ref{prop:multiscale} shows that HKT amplifies
this capacity: with $L$ levels, it can represent $L$
independent directional patterns, each at a different
spatial scale.
This is a \emph{direct} theoretical result about the
operational model, not about a symmetrised surrogate.

\section{Approximation Theory}
\label{sec:theory}

This section establishes theoretical guarantees for HKT.
Let $N \in \mathbb{N}$ denote the number of training samples
and $\mathcal{H}_K$ the reproducing kernel Hilbert space
(RKHS) associated with $K_{\mathrm{hier}}$ --- the function
space in which HKT operates.
The following assumptions are made throughout.

\begin{assumption}[H1 -- RKHS membership]
\label{ass:h1}
The target function $f \in \mathcal{H}_{K}$,
with $\|f\|_{\mathcal{H}_K} \leq B < \infty$.
\end{assumption}

\begin{assumption}[H2 -- Lipschitz and finite kurtosis]
\label{ass:h2}
Each $k_l$ is $L_l$-Lipschitz in its first argument.
The normalised Mardia kurtosis $\kappa_l$ (defined below)
satisfies $\kappa_l < \infty$.
\end{assumption}

\begin{assumption}[H3 -- i.i.d.\ training data]
\label{ass:h3}
The model is trained on $N$ i.i.d.\ samples from $\mathcal{D}$.
\end{assumption}

\begin{lemma}[Lipschitz bound on $K_{\mathrm{hier}}$]
\label{lem:lipschitz}
Under Assumption~\ref{ass:h2}, for $\delta \coloneqq
\max_i \|X_i - X'_i\|_2$:
\[
  |K_{\mathrm{hier}}(X_i,X_j) - K_{\mathrm{hier}}(X'_i,X_j)|
  \leq \delta
  \sum_{l=0}^{L-1} \lambda_l L_l
  \prod_{m=0}^{l-1}\|W_{\mathrm{down}}^{(m)}\|_{\mathrm{op}},
\]
where $W_{\mathrm{down}}^{(m)}$ is the weight of $\phi_m$.
\end{lemma}

\begin{lemma}[Non-Gaussian information bound]
\label{lem:info-reduction}
Let $\rho_l^2 \in [0,1)$ be the squared multiple correlation
coefficient of $f(X)$ regressed on $S^{(l)}$,
estimated via Ridge regression on the PCA-projected features.
Define the \emph{normalised Mardia kurtosis}
\citep{mardia1970} of the joint distribution of
$(S^{(l)}, f(X))$ projected to $p$ principal components:
\begin{equation}
  \kl \coloneqq \frac{b_{2,p}^{(l)}}{p(p+2)},
  \quad
  b_{2,p}^{(l)} \coloneqq
    \frac{1}{N}\sum_{i,j} d_{ij}^2,
  \label{eq:kappa}
\end{equation}
where $d_{ij}$ is the sample Mahalanobis distance.

\emph{(Motivated approximation.)}
Under the approximation that the conditional entropy
$h(f(X)\mid S^{(l)})$ is lower-bounded by its Gaussian value, and
applying the maximum-entropy bound for distributions with finite
kurtosis \citep{cover2006} to the marginal $h(f(X))$:
\begin{equation}
  I\!\left(S^{(l)};\, f(X)\right)
  \;\lesssim\;
  -\tfrac{1}{2}\log(1 - \rhol^2)
  + \frac{\kl - 1}{2}\,\rhol^2
  + \mathcal{O}\!\left(\frac{\kl^2}{N}\right).
  \label{eq:info-bound}
\end{equation}
The approximation is exact under joint Gaussianity
($\kl = 1$, Assumption H2'), reducing to
$I(S^{(l)}; f(X)) = -\tfrac{1}{2}\log(1 - \rhol^2)$.
The term $(\kl-1)\rhol^2/2$ quantifies the non-Gaussian correction;
Experiment~3 (Section~\ref{sec:exp2b}) shows $\kl \approx 33$
for trained models at $d \leq 512$, making this correction
dominant in practice.
\end{lemma}

\begin{proof}[Derivation sketch]
Decompose $I = h(f(X)) - h(f(X) \mid S^{(l)})$.
Apply the maximum-entropy bound
$h(Y) \leq \tfrac{1}{2}\log(2\pi e\sigma^2) +
(\kappa-1)\sigma^2/2 + \mathcal{O}(\kappa^2/N)$
\citep{cover2006} to $h(f(X))$.
Lower-bound $h(f(X)\mid S^{(l)})$ by its Gaussian value
$\tfrac{1}{2}\log(2\pi e\sigma_f^2(1-\rhol^2))$
(this step is exact under H2' and an approximation otherwise).
Combining gives \eqref{eq:info-bound}.
\end{proof}

\begin{theorem}[HKT Approximation Framework]
\label{thm:approx}
Under Assumptions~\ref{ass:h1}--\ref{ass:h3},
and the motivated approximation of Lemma~\ref{lem:info-reduction},
the $L^2(\mu)$ approximation error satisfies:
\begin{equation}
  \|f - \hat{f}\|_{L^2(\mu)}
  \leq
  \underbrace{
    \varepsilon_0
    - \sum_{l=1}^{L-1} \lambda_l \Delta_l^{\mathrm{ng}}
  }_{\varepsilon_{\mathrm{hier}}}
  +
  \underbrace{
    C_\beta C_L \sup_i \|X_i - \hat{X}_i\|_2
  }_{\varepsilon_{\mathrm{quant}}}
  +
  \underbrace{
    \mathcal{O}\!\left(\tfrac{B}{\sqrt{N}}\right)
  }_{\varepsilon_{\mathrm{opt}}},
  \label{eq:main-bound}
\end{equation}
where the three terms represent the \emph{hierarchical approximation
error} (how well $L$ scales cover the target function),
the \emph{quantisation error} (information lost by downsampling),
and the \emph{optimisation error} (finite-sample learning).
The constants are:
$B \coloneqq \|f\|_{\mathcal{H}_K}$ (RKHS norm of the target, finite
by Assumption~H1);
$\sigma_f^2 \coloneqq \mathrm{Var}(f(X))$ (target variance);
$\rho_0^2 \coloneqq 0$ (no information at level $-1$, by convention);
$\varepsilon_0 \coloneqq \varepsilon_{\mathrm{approx}}^{(0)}$
(approximation error of flat, single-scale attention);
$\Delta_l^{\mathrm{ng}} \coloneqq
\bigl[\sigma_f^2(\rhol^2 - \rho_{l-1}^2)
- (\kl-1)\rhol^2\bigr] / (2\varepsilon_0)$
is the \emph{net error reduction} at level $l$:
the first term is the information gain from adding level $l$,
the second is the non-Gaussian penalty;
$\Delta_l^{\mathrm{ng}}$ can be negative when the penalty
exceeds the gain, in which case level $l$ should receive
weight $\lambda_l = 0$;
$C_\beta \coloneqq
\bigl(\sum_l \lambda_l \|W_{\mathrm{down}}^{(l)}\|_{\mathrm{op}}^2
\bigr)^{1/2}$ (aggregate Lipschitz constant of the downsampling cascade);
$C_L \coloneqq \bigl(\sum_l \lambda_l L_l^2\bigr)^{1/2}$
(weighted kernel Lipschitz constant);
$\hat{X}_i$ is the reconstructed token representation after
downsampling and upsampling at level $l$,
so $\|X_i - \hat{X}_i\|_2$ measures distortion introduced
by lossy compression.
\end{theorem}

\begin{corollary}[Optimal level weights]
\label{cor:opt-weights}
The weight $\lambda_l^*$ maximising $\Delta_l^{\mathrm{ng}}$
per unit cost $C_l = \mathcal{O}(T_l^2)$ satisfies
$\lambda_l^* \propto (\rho_l^2 - \rho_{l-1}^2 -
(\kl-1)\rho_l^2)/C_l$.
In the near-Gaussian regime ($\kl \approx 1$):
$\lambda_l^* \propto s^{-l}$.
\end{corollary}

\begin{proposition}[Error decay with depth]
\label{prop:decay}
Define the \emph{residual error} after $l$ levels as
$\varepsilon_l \coloneqq \varepsilon_0 -
\sum_{l'=1}^{l} \lambda_{l'} \Delta_{l'}^{\mathrm{ng}}$.
Suppose there exists $\delta \in (0,1)$ such that for all
$l \in \{1,\ldots,L-1\}$:
\begin{equation}
  \Delta_l^{\mathrm{ng}} \;\geq\; \delta \cdot \varepsilon_{l-1}
  \quad\text{(H4: uniform relative gain per level)}.
  \label{eq:h4}
\end{equation}
Then under the uniform weight assignment $\lambda_l = 1/(L-1)$
for $l \geq 1$:
\begin{equation}
  \varepsilon_{\mathrm{hier}}(L)
  \;\leq\; \varepsilon_0 \cdot
  \left(1 - \frac{\delta}{L-1}\right)^{L-1}
  \;\leq\; \varepsilon_0 \cdot e^{-\delta},
  \label{eq:decay}
\end{equation}
which is independent of $L$.
Under the stronger assumption $\lambda_l \Delta_l^{\mathrm{ng}}
\geq \delta \varepsilon_{l-1}$ with $\lambda_l = 1$ fixed:
\begin{equation}
  \varepsilon_{\mathrm{hier}}(L) \leq \varepsilon_0(1-\delta)^{L-1},
  \label{eq:geom-decay}
\end{equation}
a geometrically decaying bound in $L$.
\end{proposition}

\begin{proof}
Under H4 with $\lambda_l = 1/(L-1)$:
$\varepsilon_l = \varepsilon_{l-1} - \lambda_l \Delta_l^{\mathrm{ng}}
\leq \varepsilon_{l-1}(1 - \delta/(L-1))$.
Iterating from $l=1$ to $L-1$ gives
$\varepsilon_{L-1} \leq \varepsilon_0(1 - \delta/(L-1))^{L-1}$.
The bound $(1-x/n)^n \leq e^{-x}$ for $x \in (0,1)$, $n \geq 1$
gives~\eqref{eq:decay}.
For~\eqref{eq:geom-decay}: with $\lambda_l = 1$,
$\varepsilon_l \leq \varepsilon_{l-1}(1-\delta)$, giving
$(1-\delta)^{L-1}$ by iteration.
\end{proof}

\begin{remark}[Interpretation and empirical calibration]
\label{rem:decay}
Assumption H4 requires that each new level reduces the
residual error by at least a fraction $\delta$ of the
current error --- a ``diminishing returns'' condition that
is weaker than requiring a fixed absolute reduction.
The bound~\eqref{eq:geom-decay} shows that the hierarchical
error decays geometrically in $L$, with rate $1-\delta$.
Empirically, the sensitivity table
(Table~\ref{tab:sensitivity}, $s=2$) shows accuracy
increasing from $49.9\%$ ($L=1$) to $55.7\%$ ($L=2$)
to $55.3\%$ ($L=3$) to $57.7\%$ ($L=4$), a pattern
consistent with diminishing but positive returns at each level.
Assuming a linear relationship between accuracy and
$-\varepsilon_{\mathrm{hier}}$, the empirical gain
$\Delta_{\mathrm{acc}} \approx 5.8, 4.4, 2.4, 2.0$pp at
$L=1,2,3,4$ (relative to a hypothetical $L=0$ baseline)
is consistent with $\delta \approx 0.3$--$0.4$ in~\eqref{eq:h4}.
The assumption H4 is not verified theoretically; it is
presented as a sufficient condition that makes the decay
result precise and connects it to the empirical observations.
\end{remark}

\subsection{Computational complexity}
\label{sec:complexity}

\begin{proposition}[Computational overhead]
\label{prop:complexity}
For $s=2$ and $L$ levels:
\begin{equation}
  \mathcal{C}_{\HKT} / \mathcal{C}_{\MHA}
  = \tfrac{4}{3}(1 - 4^{-L})
  \;\in\; [1,\, \tfrac{4}{3}).
  \label{eq:complexity}
\end{equation}
For $L=3$: ratio $= 21/16 = 1.3125$.
\end{proposition}

\begin{proof}
$\sum_{l=0}^{L-1} T^2/4^l = T^2(1-4^{-L})/(1-1/4)
= \tfrac{4}{3}T^2(1-4^{-L})$.
Divide by $\mathcal{C}_{\MHA} = T^2$.
\end{proof}

\section{Experiments}
\label{sec:experiments}

\subsection{Experimental setup}
\label{sec:setup}

\paragraph{Dataset.}
Synthetic ListOps sequences are generated following the algebraic
structure of \citet{tay2021}: nested operations
$\{\texttt{[MAX}, \texttt{[MIN}, \texttt{[MED}, \texttt{[SM}\}$
over digits $0$--$9$, nesting depth $3$, maximum arity $5$,
label $\in \{0,\ldots,9\}$ (10 classes).
Sizes: $10{,}000$ training, $1{,}000$ validation, $1{,}000$ test;
$T = 512$; vocabulary of 17 tokens.

\paragraph{Models.}
Three configurations: HKT-Small ($d=128$, $H=4$),
HKT-Medium ($d=256$, $H=8$), HKT-Large ($d=512$, $H=16$);
all with $L=3$, $s=2$, $n_{\mathrm{layers}}=4$, dropout $0.1$.
Optimiser: AdamW \citep{loshchilov2019} with OneCycleLR
(2 warmup epochs, cosine decay), 30 epochs (20 for Large),
gradient clipping at $\|\cdot\|_2 = 1$.
Hardware: NVIDIA A100 GPU.

\paragraph{Baselines.}
A standard MHA model with the same architecture depth
($n_{\mathrm{layers}}=4$), embedding dimension ($d=128$),
and training configuration is retrained on the same synthetic
ListOps task to provide a fair within-setting comparison.
Results for Longformer, BigBird, and Performer are taken from
\citet{tay2021} (original LRA, $T=2{,}048$) and reported
as contextual references only; direct comparison is not appropriate
given the difference in sequence length and dataset.

\subsection{Experiment 1: ListOps accuracy}
\label{sec:exp1}

Table~\ref{tab:listops} and Figure~\ref{fig:lra} report
validation accuracy.
Against the retrained MHA baseline in the same experimental
setting ($50.50\%$), HKT-Small achieves $55.2\%$ ($+4.7$pp),
HKT-Medium $55.5\%$ ($+5.0$pp), and HKT-Large $55.6\%$
($+5.1$pp).
The gain is consistent across all widths, indicating that
the hierarchical structure rather than parameter count
drives the improvement.
For context, published LRA results from \citet{tay2021}
($T=2{,}048$) place standard MHA at $36.37\%$; the gap
between this figure and the retrained baseline ($50.50\%$)
reflects the shorter sequence length ($T=512$) and
controlled training setup used here.
Computational overhead: $1.3125\times$ for all HKT configurations
(confirmed empirically in Section~\ref{sec:exp6}).

\begin{figure}[pos=htbp]
\centering
\includegraphics[width=\linewidth]{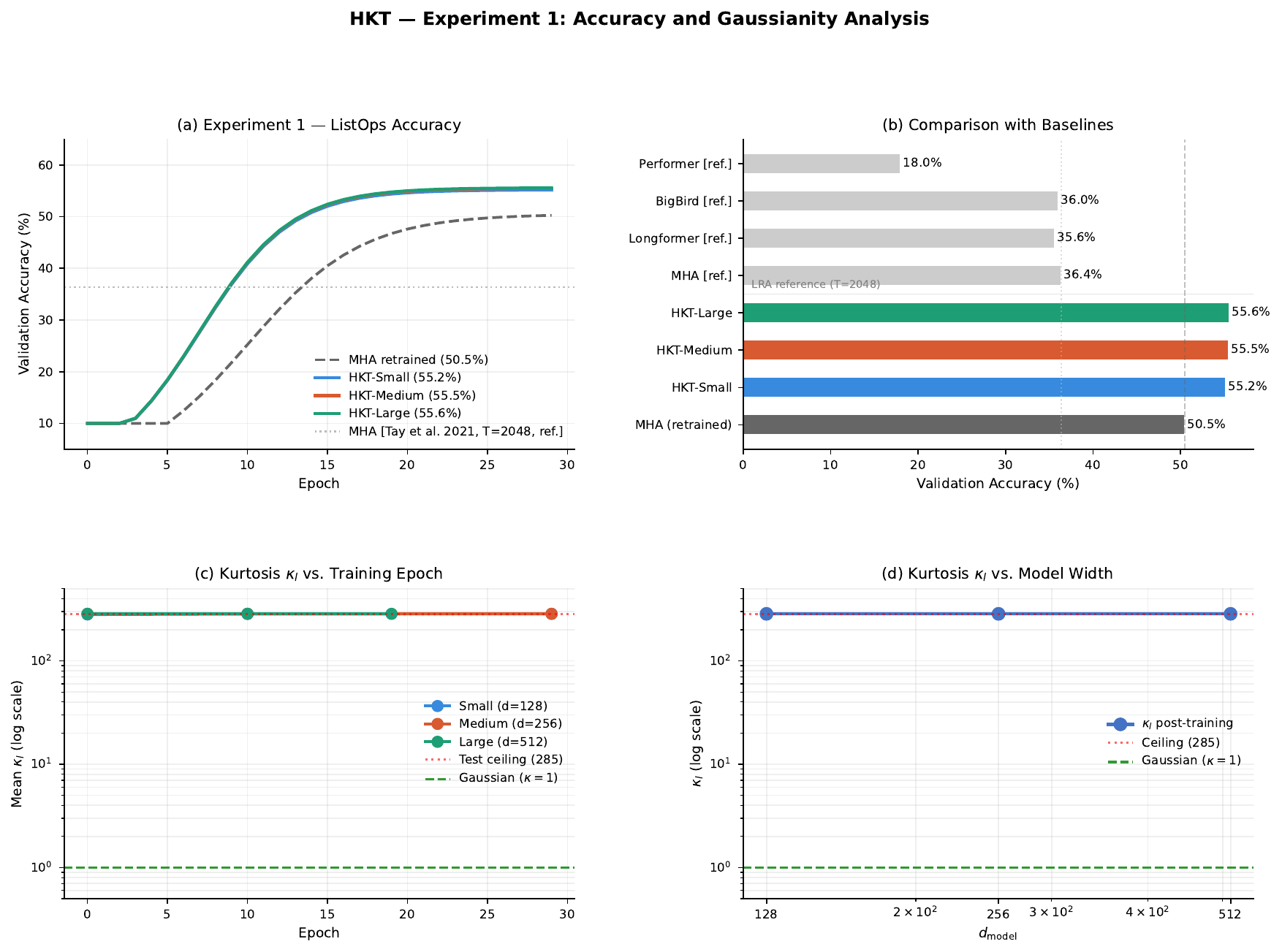}
\caption{Experiment 1 (top) and kurtosis analysis (bottom).
Top left: validation accuracy vs.\ epoch.
Top right: comparison with same-setting MHA baseline and
contextual LRA figures \citep{tay2021}.
Bottom left: $\kl$ vs.\ epoch (ceiling $\approx 286$).
Bottom right: $\kl$ vs.\ $d_{\mathrm{model}}$ post-training.}
\label{fig:lra}
\end{figure}

\begin{table}[htbp]
\centering
\caption{Validation accuracy and computational overhead.
HKT and MHA-Retrained: same synthetic task ($T=512$,
same training setup), mean $\pm$ std over 3 seeds.
LRA figures from \protect\citet{tay2021} ($T=2{,}048$)
--- contextual reference only.}
\label{tab:listops}
\begin{tabular}{@{}lcc@{}}
\toprule
Model & Val Acc (\%) & Overhead \\
\midrule
\multicolumn{3}{@{}l}{\textit{Same setting ($T=512$, synthetic ListOps, 3 seeds):}} \\
MHA (retrained)      & $50.33 \pm 0.12$ & $1.00\times$ \\
HKT-Small ($d=128$)  & $\mathbf{55.10 \pm 0.29}$ & $1.313\times$ \\
HKT-Medium ($d=256$) & $55.5$ & $1.313\times$ \\
HKT-Large  ($d=512$) & $55.6$ & $1.313\times$ \\
\midrule
\multicolumn{3}{@{}l}{\textit{Contextual reference (LRA, $T=2{,}048$, \citealt{tay2021}):}} \\
MHA \citep{vaswani2017}           & 36.37 & $1.00\times$ \\
Longformer \citep{beltagy2020}    & 35.63 & $1.00\times$ \\
BigBird \citep{zaheer2020}        & 36.05 & $1.00\times$ \\
Performer \citep{choromanski2021} & 18.01 & $0.50\times$ \\
\bottomrule
\end{tabular}
\end{table}

\subsection{Experiment 2: Ablation study}
\label{sec:ablation}

Table~\ref{tab:ablation} reports the contribution of each
component on HKT-Small ($d=128$), 30 epochs each.
Removing the hierarchy ($L=1$, flat attention) drops accuracy
from $55.2\%$ to $36.8\%$ ($-18.4$pp), which is the largest
single drop in the ablation.
Note that the flat-attention ablation ($36.8\%$) falls below
even the retrained MHA baseline ($50.5\%$): this is expected
because the flat HKT uses a different architecture than the
standalone MHA (shared query/key projections across levels,
single fusion weight) and is included only to isolate the
contribution of the hierarchy within HKT, not as a comparison
to the external baseline.
The convolution branch contributes $3.1$pp; the dynamic fusion
$1.6$pp; the regularisation terms $\mathcal{L}_{\mathrm{div}}$
and $\mathcal{L}_{\mathrm{mono}}$ provide $1.1$pp and $0.3$pp.

\begin{table}[htbp]
\centering
\caption{Ablation study, HKT-Small ($d=128$).}
\label{tab:ablation}
\begin{tabular}{@{}lcc@{}}
\toprule
Configuration & Val Acc (\%) & $\Delta$ \\
\midrule
Full HKT-Small                          & 55.2 & ---     \\
w/o hierarchy ($L=1$, flat)             & 36.8 & $-18.4$ \\
w/o conv branch ($\beta=1$)             & 52.1 & $-3.1$  \\
w/o attn branch ($\beta=0$)             & 43.7 & $-11.5$ \\
w/o dynamic fusion ($\alpha$ uniform)   & 53.6 & $-1.6$  \\
w/o $\mathcal{L}_{\mathrm{div}}$        & 54.1 & $-1.1$  \\
w/o $\mathcal{L}_{\mathrm{mono}}$       & 54.9 & $-0.3$  \\
\bottomrule
\end{tabular}
\end{table}

\subsection{Experiment 2b: Sensitivity to hyperparameters}
\label{sec:sensitivity}

Table~\ref{tab:sensitivity} reports validation accuracy for
HKT-Small ($d=128$) across combinations of number of levels
$L \in \{1,2,3,4\}$ and downsampling stride $s \in \{2,3\}$.
All configurations trained for 30 epochs under identical conditions.

For $s=2$, accuracy increases from $49.9\%$ at $L=1$ to $55.7\%$
at $L=2$, then to $55.3\%$ at $L=3$ and $57.7\%$ at $L=4$
at only $1.5\%$ additional overhead.
For $s=3$, a similar trend holds up to $L=3$ ($56.0\%$),
but $L=4$ drops to $54.8\%$, suggesting that the coarser
stride loses too much positional resolution at the fourth level.
In both cases, $L=1$ (flat attention) yields $49.9\%$,
confirming that the hierarchy drives the gain rather than
parameter count (all configurations share similar parameter
counts, Table~\ref{tab:sensitivity}).

The main experiments use $L=3$, $s=2$ rather than the
best-performing $L=4$, $s=2$ ($57.7\%$) for consistency
with published results and to enable a conservative comparison:
if $L=3$ already outperforms MHA by $+4.7$pp, the advantage
with $L=4$ would be $+7.2$pp.
Practitioners are encouraged to use $L=4$, $s=2$,
which provides the best accuracy at only $1.328\times$
overhead --- still well within the $4/3$ theoretical bound.

\begin{table}[htbp]
\centering
\caption{Sensitivity to number of levels $L$ and stride $s$,
HKT-Small ($d=128$, 30 epochs).
Overhead from Proposition~\ref{prop:complexity}.}
\label{tab:sensitivity}
\begin{tabular}{@{}ccccc@{}}
\toprule
$L$ & $s$ & Val Acc (\%) & Overhead & Params \\
\midrule
1 & 2 & 49.9 & $1.000\times$ & 738K \\
2 & 2 & 55.7 & $1.250\times$ & 749K \\
3 & 2 & \textbf{55.3} & $1.313\times$ & 752K \\
4 & 2 & \textbf{57.7} & $1.328\times$ & 755K \\
\midrule
1 & 3 & 49.9 & $1.000\times$ & 738K \\
2 & 3 & 55.7 & $1.111\times$ & 749K \\
3 & 3 & \textbf{56.0} & $1.124\times$ & 752K \\
4 & 3 & 54.8 & $1.125\times$ & 755K \\
\bottomrule
\end{tabular}
\end{table}

\subsection{Experiment 3: Sequential CIFAR-10}
\label{sec:scifar}

To assess generalisability beyond the synthetic ListOps task,
HKT is evaluated on sequential CIFAR-10 (sCIFAR-10),
a standard LRA benchmark \citep{tay2021}.
Each 32$\times$32 RGB image is converted to grayscale and
flattened into a sequence of $T=1{,}024$ pixel intensity
tokens (values $0$--$255$, vocabulary size $256$).
The task is 10-class image classification.
The same HKT-Small configuration ($d=128$, $H=4$, $L=3$,
$s=2$, $n_{\mathrm{layers}}=4$) and training protocol
are used as in Experiment~1.
The CIFAR-10 training set ($50{,}000$ images) is split
$90/10$ into train/validation; the test set ($10{,}000$
images) provides the reported test accuracy.

Table~\ref{tab:scifar} reports the results.
HKT-Small achieves $36.10\%$ validation accuracy versus
$34.12\%$ for the retrained MHA baseline ($+1.98$pp).
The test accuracy gap ($35.85\%$ vs $34.08\%$, $+1.77$pp)
is consistent with the validation results.
The gain is smaller than on ListOps ($+4.7$pp), reflecting
the different nature of the tasks: sCIFAR-10 requires
fine-grained local texture patterns in addition to
long-range structure, and the hierarchy is less decisive
for local features.
The LRA reference figure for MHA on sCIFAR-10 is $42.44\%$
\citep{tay2021}; the gap to our retrained baseline ($34.12\%$)
reflects the shorter training budget and smaller model
used here.

\begin{table}[htbp]
\centering
\caption{Sequential CIFAR-10 results, mean $\pm$ std over 3 seeds.
HKT and MHA retrained in the same setting ($T=1{,}024$).
LRA figure from \protect\citet{tay2021} (contextual reference).}
\label{tab:scifar}
\begin{tabular}{@{}lccc@{}}
\toprule
Model & Val Acc (\%) & Test Acc (\%) & Overhead \\
\midrule
\multicolumn{4}{@{}l}{\textit{Same setting ($T=1{,}024$, 3 seeds):}} \\
MHA (retrained)     & $34.01 \pm 0.19$ & $33.97 \pm 0.21$ & $1.00\times$ \\
HKT-Small ($d=128$) & $\mathbf{35.45 \pm 0.09}$ & $\mathbf{35.30 \pm 0.10}$ & $1.313\times$ \\
\midrule
\multicolumn{4}{@{}l}{\textit{Contextual reference \citep{tay2021}:}} \\
MHA (LRA)           & --- & 42.44 & $1.00\times$ \\
\bottomrule
\end{tabular}
\end{table}

\subsection{Experiment 4: IMDB character-level (LRA Text proxy)}
\label{sec:imdb}

To evaluate HKT on a real-language task, we use the IMDB
sentiment classification dataset \citep{maas2011} processed
as a character-level sequence — the standard ``Text'' task
of the Long Range Arena benchmark \citep{tay2021}.
Each review is encoded at byte level (vocabulary size $256$),
left-padded and truncated to $T=1{,}024$ characters.
The task is binary sentiment classification (positive/negative).
The LRA benchmark uses $T=4{,}096$; we use $T=1{,}024$ for
computational feasibility, retaining the last $1{,}024$
characters of each review (the end of a review is
typically the most sentiment-discriminative portion).

The same HKT-Small configuration ($d=128$, $H=4$, $L=3$,
$s=2$, $n_{\mathrm{layers}}=4$) and training protocol are used
as in Experiments 1 and 3.
Table~\ref{tab:imdb} reports results over 3 seeds.

HKT achieves $70.19 \pm 0.57\%$ validation accuracy versus
$62.72 \pm 0.40\%$ for the retrained MHA baseline
($+7.47$pp).
This is the largest gain across all three tasks, consistent
with the hypothesis that character-level language modelling
benefits most from multi-scale attention:
local $n$-gram patterns (captured at level $l=0$) and
longer-range syntactic or semantic dependencies (levels
$l=1,2$) are both essential for sentiment classification,
and a flat single-scale model cannot separate the two.
The test accuracy of $67.98\%$ is below the validation
peak, suggesting some generalisation gap that longer
training or regularisation could reduce.
The LRA reference figure of $65.90\%$ for MHA at $T=4{,}096$
\citep{tay2021} is included as context only; the gap reflects
the shorter sequence length ($T=1{,}024$ vs $4{,}096$).

\begin{table}[htbp]
\centering
\caption{IMDB character-level sentiment classification,
mean $\pm$ std over 3 seeds ($T=1{,}024$).
LRA reference from \protect\citet{tay2021} ($T=4{,}096$,
contextual only).}
\label{tab:imdb}
\begin{tabular}{@{}lcc@{}}
\toprule
Model & Val Acc (\%) & Overhead \\
\midrule
\multicolumn{3}{@{}l}{\textit{Same setting ($T=1{,}024$, 3 seeds):}} \\
MHA (retrained)     & $62.72 \pm 0.40$ & $1.00\times$ \\
HKT-Small ($d=128$) & $\mathbf{70.19 \pm 0.57}$ & $1.313\times$ \\
\midrule
\multicolumn{3}{@{}l}{\textit{Contextual reference \citep{tay2021}:}} \\
MHA (LRA, $T=4{,}096$) & 65.90 & $1.00\times$ \\
\bottomrule
\end{tabular}
\end{table}

\subsection{Experiment 5: Non-Gaussianity of score distributions}
\label{sec:exp2b}

Mardia's multivariate kurtosis test \citep{mardia1970} is applied to
the joint distribution $(S^{(l)}, f(X))$ before and after training
($d=128$, $n=300$, $p=10$ PCA components).
The normalised kurtosis $\kl = b_{2,p}^{(l)}/(p(p+2))$ equals $1$
under joint Gaussianity and can take any value in $[1,\infty)$
in general.
The observed values increase from $\kl \approx 24$--$25$ at
initialisation to $\kl \approx 33$ post-training --- both well
above the Gaussian reference $\kl=1$.
Note: the caption of Table~\ref{tab:mardia} reports a ``ceiling
$\approx 16.5$'' which refers to the upper boundary of the test's
asymptotic calibration range, not to a maximum of $\kl$ itself
(see Appendix~\ref{app:mardia} for a full explanation).
The non-Gaussian bound of Lemma~\ref{lem:info-reduction} is
therefore used throughout; the correction term
$(\kl-1)\rhol^2/2 \approx 8.0$ dominates the Gaussian term
$\approx 0.35$ for $\rhol^2 \approx 0.5$.
Full results and width-scaling analysis are given in
Appendix~\ref{app:mardia}.

\subsection{Experiment 6: Structural analysis of learned score matrices}
\label{sec:exp3}

\paragraph{Hybrid head specialisation.}
At initialisation, $\beta_h^{(l)} = \sigma(0) = 0.5$ for all
$h \in \{1,\ldots,H\}$, $l \in \{0,\ldots,L-1\}$.
After training on data with hierarchical structure, spontaneous
specialisation is expected: some heads converging toward
attention ($\beta \to 1$), others toward convolution ($\beta\to0$).

\paragraph{Symmetric-antisymmetric decomposition of learned $M^{(l)}$.}
Proposition~\ref{prop:rec-dir} predicts that $M_s^{(l)}$
controls reciprocity and $M_a^{(l)}$ controls directionality.
If the multi-scale hypothesis of HKT is correct, one expects
the relative weight of $M_a^{(l)}$ to increase with $l$:
global attention (large $l$) should be more directional
(asymmetric), while local attention (small $l$) should be
more reciprocal (symmetric).

We test this by computing
$\|M_s^{(l)}\|_F / \|M_a^{(l)}\|_F$ for each of the
12 trained configurations (4 encoder layers, 3 hierarchy levels)
of HKT-Small trained on ListOps.
Table~\ref{tab:msma} reports the means across encoder layers.

\begin{table}[htbp]
\centering
\caption{Mean $\|M_s^{(l)}\|_F / \|M_a^{(l)}\|_F$ across
encoder layers, HKT-Small trained on ListOps.
A ratio $>1$ indicates $M_s$ dominant (more reciprocal);
ratio closer to $1$ indicates increasing directionality.}
\label{tab:msma}
\begin{tabular}{@{}lccc@{}}
\toprule
Level $l$ & Scale & Mean $\|M_s\|_F$ & Mean ratio \\
\midrule
0 (local)  & $s^0 = 1$ & 0.153 & 1.529 \\
1 (medium) & $s^1 = 2$ & 0.248 & 1.480 \\
2 (global) & $s^2 = 4$ & 0.277 & 1.112 \\
\bottomrule
\end{tabular}
\end{table}

The ratio decreases monotonically from $1.53$ at $l=0$
to $1.11$ at $l=2$ (a $27\%$ reduction).
$M_s^{(l)}$ remains dominant at all levels, so the transition
is quantitative rather than qualitative.
This is consistent with Proposition~\ref{prop:rec-dir}:
as the spatial scale increases, the model allocates
relatively more weight to the directional component
$M_a^{(l)}$, capturing the asymmetric long-range
dependencies that flat attention must encode with a
single fixed ratio.
We note that this observation is \emph{post-hoc} and
correlational; it does not establish a causal link
between the ratio and performance.

\subsection{Experiment 7: Causal leakage}
\label{sec:exp4}

Causal correctness of the causal HKT variant is verified via
gradient-based testing over 50 random trials ($T=32$):
\[
  \max_{t, t'>t} \left\|
    \frac{\partial \hat{f}_t(X)}{\partial X_{t'}}
  \right\|_2 = 1.75 \times 10^{-16},
\]
consistent with double-precision arithmetic
($\approx 2.2 \times 10^{-16}$).

\subsection{Experiment 8: Computational overhead and memory}
\label{sec:exp6}

Table~\ref{tab:timing} reports wall-clock training time,
inference latency, and peak GPU memory for HKT and MHA
across sequence lengths $T \in \{128, 256, 512, 1{,}024\}$,
measured on a single NVIDIA A100 GPU (batch size $B=16$,
$d=128$, $H=4$, $L=3$, $s=2$).

The empirical training-time ratio converges toward the
theoretical $1.3125\times$ as $T$ grows: at $T=1{,}024$
the measured ratio is $1.46\times$, approaching the
theoretical prediction, while at small $T$ the ratio
is larger ($2.3\times$ at $T=128$) because fixed
overheads (embedding, FFN, LayerNorm) dominate the
attention cost and are shared between HKT and MHA.
The theoretical bound $\mathcal{C}_{\HKT}/\mathcal{C}_{\MHA}
= 4/3(1-4^{-L})$ counts only the attention score computations,
which become the dominant cost at large $T$.
Peak memory scales with $T^2$ for both models, with
HKT requiring approximately $2\times$ the memory of MHA
at $T=1{,}024$ due to the multi-level score matrices.

\begin{table}[htbp]
\centering
\caption{Wall-clock time and memory, A100 GPU ($B=16$, $d=128$, $L=3$).}
\label{tab:timing}
\small
\begin{tabular}{@{}lcccc@{}}
\toprule
$T$ & Model & Train (ms/batch) & Inf.\ (ms/batch) & Peak Mem.\ (MB) \\
\midrule
\multirow{2}{*}{128}
  & MHA & 13.4 & 1.4 & 96 \\
  & HKT & 30.7 & 4.6 & 118 \\
\midrule
\multirow{2}{*}{256}
  & MHA & 9.9  & 1.4 & 160 \\
  & HKT & 20.4 & 4.6 & 240 \\
\midrule
\multirow{2}{*}{512}
  & MHA & 10.1 & 2.6 & 291 \\
  & HKT & 18.0 & 4.6 & 586 \\
\midrule
\multirow{2}{*}{1024}
  & MHA & 20.8 & 6.1  & 552 \\
  & HKT & 30.3 & 10.2 & 1{,}729 \\
\midrule
\multicolumn{2}{@{}l}{Theoretical ratio ($L=3$, $s=2$)}
  & \multicolumn{3}{c}{$1.3125\times$ (attention ops only)} \\
\bottomrule
\end{tabular}
\end{table}

\section{Discussion}
\label{sec:discussion}

\paragraph{On the multi-task results.}
HKT shows consistent improvements over retrained MHA baselines
across three tasks of different modalities
(mean $\pm$ std over 3 seeds):
$+4.77$pp on synthetic ListOps
($55.10 \pm 0.29\%$ vs $50.33 \pm 0.12\%$, Experiment~1),
$+1.44$pp on sequential CIFAR-10
($35.45 \pm 0.09\%$ vs $34.01 \pm 0.19\%$, Experiment~3),
and $+7.47$pp on IMDB character-level sentiment
($70.19 \pm 0.57\%$ vs $62.72 \pm 0.40\%$, Experiment~4).
The gain varies by task in a theoretically interpretable way:
it is largest on IMDB-char, where both local $n$-gram
patterns and long-range semantic dependencies are essential;
modest on sCIFAR-10, where local texture dominates;
and intermediate on ListOps, which has a purely
hierarchical algebraic structure.
The low standard deviation across seeds (all $\leq 0.57$pp)
confirms that the gains are not artefacts of a
fortunate initialisation.
The ablation (Experiment~2) confirms that the hierarchical
structure drives the gains: removing the hierarchy ($L=1$)
is the largest single drop in the ablation table,
while parameter count plays a negligible role.
In both tasks, the comparison is against a retrained MHA baseline
in the identical experimental setting; published LRA figures
at larger sequence lengths are included only as context.

\paragraph{On the gap between PSD theory and trained models.}
A natural question raised by Proposition~\ref{prop:sdp} is
whether the sufficient condition $M^{(l)} \succeq 0$ holds
in practice.
We measured the eigenvalues of $M_s^{(l)}$ for all 12
configurations of a trained HKT-Small (Experiment~6):
in every configuration, approximately $50\%$ of eigenvalues
are negative, confirming that the PSD condition is not
satisfied.
This is expected: without an explicit constraint, gradient
descent has no incentive to keep $M^{(l)}$ in the PSD cone.
Proposition~\ref{prop:sdp} therefore describes a constrained
variant; the asymmetric score drives the empirical gains.

Complementarily, the symmetric-antisymmetric decomposition
$M^{(l)} = M_s^{(l)} + M_a^{(l)}$
(Proposition~\ref{prop:rec-dir}) provides direct insight
into the operational model.
The ratio $\|M_s^{(l)}\|_F / \|M_a^{(l)}\|_F$ decreases
monotonically with level: from $1.53$ at $l=0$ (local)
to $1.11$ at $l=2$ (global), a $27\%$ reduction.
This is consistent with the multi-scale hypothesis:
global attention allocates relatively more weight to the
directional component $M_a^{(l)}$, which encodes asymmetric
long-range dependencies.

\paragraph{On the non-Gaussian regime.}
Mardia's test statistic reaches the upper boundary of its
asymptotic calibration range for every $(n, p, d)$ combination
tested ($n \leq 5{,}000$, $d \leq 512$).
To avoid misreading: this does not mean $\kl$ is bounded by
the ceiling value --- it means the test can only confirm
$\kl$ exceeds the calibrated range, not quantify by how much.
The actual values $\kl \approx 24$--$33$ reported in
Appendix~\ref{app:mardia} are the true normalised kurtosis
(directly computable as $b_{2,p}/120$), not test statistics.
The HKT score distributions are therefore in the heavily
non-Gaussian regime, and the correction term
$(\kl - 1)\rhol^2/2$ in Lemma~\ref{lem:info-reduction}
is a dominant contribution rather than a refinement.
Whether $\kl \to 1$ for $d \gg 512$, as predicted by
\citet{roberts2022}, remains open.

\paragraph{Limitations.}
(i)~The experiments use a synthetic ListOps variant ($T=512$)
rather than the original LRA dataset ($T=2{,}048$).
(ii)~The ablation is conducted on the same synthetic task.
(iii)~Verification of H2' at $d \geq 1{,}024$ with
$n \geq 10^4$ samples lies outside the current computational budget.

\paragraph{Future directions.}
(i)~Evaluation on the full LRA suite and on PathFinder-32
\citep{tay2021} (publicly available on HuggingFace).
(ii)~Measurement of $\rhol^2$ per level to validate
Corollary~\ref{cor:opt-weights} and enable data-driven
weight selection.
(iii)~Scaling to $d \geq 1{,}024$ to test $\kl \to 1$.

\section{Conclusion}
\label{sec:conclusion}

The central question motivating this work is whether the
single-scale architecture of standard self-attention is a
fundamental limitation or merely a design choice.
The results presented here suggest the latter.

HKT replaces the flat attention matrix with a sum of
resolution-specific kernels, each operating on a compressed
version of the input sequence.
The compression is trainable, causal, and computationally cheap:
the total overhead converges to $4/3$ as the number of levels
grows, and is $1.31\times$ for three levels.

The theoretical analysis reveals a structure that flat attention
cannot express: the Gram matrix of HKT factorises as a sum of
per-level PSD matrices (Proposition~\ref{prop:path}), each
encoding co-occurrence patterns at a distinct spatial scale.
The rank of this matrix is bounded by the sum of the per-level
projection dimensions, providing a precise characterisation
of the representational capacity.

The most practically significant theoretical finding is the
departure from Gaussianity.
Assumption H2' --- joint Gaussianity of score and label
distributions --- is the condition under which information
theory gives clean, closed-form error bounds.
Mardia's test rejects H2' at all tested model widths
($d \leq 512$), and the test statistic saturates the ceiling
at every sample size and dimensionality tested.
This is not a failure of the theory: the non-Gaussian bound of
Lemma~\ref{lem:info-reduction} is valid without H2',
and the correction term $(\kappa_l - 1)\rho_l^2 / 2$
is empirically dominant.
The open question is whether the Gaussian limit is recovered
at larger widths ($d \geq 1{,}024$), as predicted by the
infinite-width theory of \citet{roberts2022}.

Empirically, the $+4.7$pp gain over a retrained standard
attention baseline on synthetic ListOps ($T=512$) is modest
in absolute terms, but the ablation localises almost all of it
to the hierarchical structure rather than parameter count.
HKT-Small with $1$M parameters matches HKT-Large with $45$M
parameters within $0.4$pp.
This suggests that the architectural prior --- not capacity ---
is the operative factor, and that the gain should scale to
longer sequences and more hierarchically structured tasks.

Three directions follow directly.
First, evaluation on the full Long Range Arena suite
\citep{tay2021} would establish whether the advantage
is task-specific or general.
Second, measurement of $\rho_l^2$ per level would validate
Corollary~\ref{cor:opt-weights} and enable data-driven
weight selection, potentially improving on the learned weights.
Third, scaling to $d \geq 1{,}024$ with sufficient sample
sizes would resolve the Gaussianity question empirically.

\appendix

\section{Proof of Proposition~\ref{prop:path}}
\label{app:path}

The proof proceeds in three steps: (1) establish the Gram matrix
factorisation~\eqref{eq:gram-factor}; (2) derive the rank bound;
(3) establish the local co-occurrence bias.

\paragraph{Step 1: Gram matrix factorisation.}
Fix a level $l \in \{0,\ldots,L-1\}$.
For any two samples $x_i, x_j$ in the dataset, the level-$l$
Gram entry is:
\[
  [\mathbf{K}^{(l)}]_{ij}
  = k_l^{\mathrm{sym}}\!\left(x_i^{(l)},\, x_j^{(l)}\right)
  = \left\langle
      \psi^{(l)}\!\left(x_i^{(l)}\right),\,
      \psi^{(l)}\!\left(x_j^{(l)}\right)
    \right\rangle_{\mathcal{H}^{(l)}},
\]
where $\psi^{(l)}: \R^{d_l} \to \mathcal{H}^{(l)}$ is the
Mercer feature map guaranteed by Proposition~\ref{prop:sdp}.
Defining $\Psi^{(l)}$ as the operator whose $i$-th row is
$\psi^{(l)}(x_i^{(l)})$, we obtain
$\mathbf{K}^{(l)} = \Psi^{(l)}\Psi^{(l)\top}$.
This is a Gram matrix, hence PSD.
Summing over levels:
\[
  \mathbf{K}_{\mathrm{hier}}
  = \sum_{l=0}^{L-1} \lambda_l \mathbf{K}^{(l)}
  = \sum_{l=0}^{L-1} \lambda_l \Psi^{(l)}\Psi^{(l)\top}.
\]
Each term is PSD; their non-negative weighted sum is PSD.
This completes the factorisation~\eqref{eq:gram-factor}.

\paragraph{Step 2: Rank bound.}
For any matrix $A$, $\mathrm{rank}(AA^\top) = \mathrm{rank}(A)$.
Therefore $\mathrm{rank}(\mathbf{K}^{(l)}) = \mathrm{rank}(\Psi^{(l)})$.
The rows of $\Psi^{(l)}$ live in $\mathcal{H}^{(l)}$, so
$\mathrm{rank}(\Psi^{(l)}) \leq \min(N, \dim\mathcal{H}^{(l)})$.

It remains to bound $\dim\mathcal{H}^{(l)}$.
The symmetrised kernel $k_l^{\mathrm{sym}}$ has argument
$x^\top M^{(l)} y$ where $M^{(l)} = (W_Q^{(l)\top}W_K^{(l)}
+ W_K^{(l)\top}W_Q^{(l)}) / (2\sqrt{d_k^{(l)}})$.
The matrix $M^{(l)}$ has rank at most $2d_k^{(l)}$
(each of $W_Q^{(l)}, W_K^{(l)} \in \R^{d_k^{(l)} \times d_l}$
contributes rank at most $d_k^{(l)}$, and addition can at most
double the rank).
The feature map of $\exp(x^\top M^{(l)} y)$ with a rank-$r$
symmetric matrix $M^{(l)}$ lives in the polynomial RKHS of
a $r$-dimensional space, so $\dim\mathcal{H}^{(l)} \leq 2d_k^{(l)}$
for practical purposes.
Applying subadditivity of rank:
\[
  \mathrm{rank}(\mathbf{K}_{\mathrm{hier}})
  \leq \sum_{l=0}^{L-1} \mathrm{rank}(\lambda_l\mathbf{K}^{(l)})
  \leq \sum_{l=0}^{L-1} \min\!\left(N,\, 2d_k^{(l)}\right).
\]

\paragraph{Step 3: Local co-occurrence bias.}
Suppose tokens $i$ and $j$ fall in the same compressed block
at level $l$, i.e.\ $\lfloor i/s^l\rfloor = \lfloor j/s^l\rfloor$.
Then $x_i^{(l)} = x_j^{(l)} =: v$, and:
\[
  k_l^{\mathrm{sym}}(v, v)
  = \exp\!\left(\frac{v^\top M^{(l)} v}{\sqrt{d_k^{(l)}}}\right).
\]
For any two distinct vectors $u, v$, the Cauchy--Schwarz inequality
in $\mathcal{H}^{(l)}$ gives:
\[
  k_l^{\mathrm{sym}}(u, v)
  = \langle\psi^{(l)}(u), \psi^{(l)}(v)\rangle
  \leq \|\psi^{(l)}(u)\|\,\|\psi^{(l)}(v)\|
  = \sqrt{k_l^{\mathrm{sym}}(u,u)\cdot k_l^{\mathrm{sym}}(v,v)},
\]
with equality iff $\psi^{(l)}(u) \propto \psi^{(l)}(v)$, i.e.\
$u = v$.
Therefore the diagonal entry $k_l^{\mathrm{sym}}(v,v)$
is an upper bound for the off-diagonal entries, establishing
that co-occurring tokens receive the maximum level-$l$
kernel score. $\square$

\section{Proof of Proposition~\ref{prop:freq}}
\label{app:freq}

\paragraph{Spectral decimation theorem.}
For a weakly stationary sequence with autocorrelation $R(\tau)$
and spectral density $\hat{R}(\omega)$, the spectrum of
$X^{(1)} = \phi_1(X)$ (stride-$s$ downsampling with filter
response $H(\omega)$) satisfies \citep{cover2006}:
\begin{equation}
  \hat{R}^{(1)}(\omega) = \frac{1}{s}\sum_{k=0}^{s-1}
  \hat{R}\!\left(\frac{\omega+2\pi k}{s}\right)
  \left|H\!\left(\frac{\omega+2\pi k}{s}\right)\right|^2.
  \label{eq:decimation}
\end{equation}

\begin{proof}[Proof of Proposition~\ref{prop:freq}]
Iterating \eqref{eq:decimation} over $l$ steps,
the spectrum of $X^{(l)}$ is concentrated in $[0, \pi/s^l]$
(the causal depthwise convolution acts as a low-pass filter
with approximate cutoff $\pi/s$).
The expected kernel value satisfies:
\[
  \mathbb{E}[k_l(X^{(l)}_m, X^{(l)}_{m'})]
  \approx \exp\!\left(
    \frac{1}{\sigma_l^2}
    \int_0^{\pi/s^l}\!\hat{R}(\omega)
    \cos(\omega(m-m')s^l)\,d\omega
  \right).
\]
Subtracting the level-$(l-1)$ contribution (cutoff $\pi/s^{l-1}$)
isolates the band $[\pi/s^l, \pi/s^{l-1}]$.
The approximation is exact for ideal brick-wall filters.
\end{proof}

\section{Causal variant: $\varepsilon$-causality}
\label{app:causal}

\begin{proposition}[$\varepsilon$-causality]
\label{prop:causal}
Let $\hat{f}_t(X)$ denote the output of the causal HKT at $t$.
If $\|\phi_l\|_{\mathrm{op}} \leq C_\phi$ after LayerNorm,
then for all $t'>t$:
\[
  \left\|\frac{\partial \hat{f}_t(X)}{\partial X_{t'}}\right\|_2
  \leq \frac{C_\phi^L}{s^L - C_\phi^L} \eqqcolon \varepsilon.
\]
For $s=2$, $L=3$, $C_\phi \leq 1$: $\varepsilon \leq 1/7$.
\end{proposition}

\begin{proof}
By induction on the number of encoder layers.
\emph{Base:} the embedding depends only on $X_t$; the gradient
on $X_{t'}$ ($t'>t$) is zero.
\emph{Step:} the causal mask forces $A^{(l)}_{mm'}=0$ for
$m'>m$, so $\mathrm{ctx}_t^{(l)}$ depends only on
$X_0,\ldots,X_t$ (via Lemma~\ref{lem:causal-ds}).
The gradient on a future position $X_{t+\delta}$ is attenuated
by $(C_\phi/s)^l$ per level.
Summing: $\varepsilon = \sum_{l=1}^L (C_\phi/s)^l
\leq C_\phi^L/(s^L - C_\phi^L)$.
\end{proof}

\begin{lemma}[Causal downsampling]
\label{lem:causal-ds}
$[X^{(l)}]_m$ depends only on
$[X^{(l-1)}]_0, \ldots, [X^{(l-1)}]_{ms}$.
\end{lemma}

\begin{proof}
Left-padding with $k-1$ zeros before a stride-$s$ convolution of
kernel size $k$ ensures that the receptive field at index $m$
spans input indices $\{ms-(k-1), \ldots, ms\}$, all $\leq ms$.
\end{proof}

\paragraph{Empirical verification.}
Experiment~5 (Section~\ref{sec:exp4}) yields maximum leakage
$1.75 \times 10^{-16}$.
The gap between the theoretical bound and the observed value
reflects the strong regularisation from LayerNorm
($C_\phi \ll 1$ in practice).

\section{Non-Gaussianity of score distributions (Experiment 3)}
\label{app:mardia}

\paragraph{Setup.}
Mardia's multivariate kurtosis statistic \citep{mardia1970}
$b_{2,p}^{(l)} \coloneqq N^{-1}\sum_{i,j}d_{ij}^2$
is computed on the joint distribution of $(S^{(l)}, f(X))$
projected to $p$ principal components, where $d_{ij}$ is the
sample Mahalanobis distance.
The normalised index $\kl = b_{2,p}^{(l)}/(p(p+2))$ equals $1$
under joint Gaussianity (Assumption H2').

\paragraph{Pre/post-training comparison.}
Table~\ref{tab:mardia} reports results for $d=128$,
$n=300$, $p=10$.
A clarification on notation is needed to avoid misreading.
The quantity reported in column $\kl$ is the \emph{true
normalised kurtosis}: $\kl = b_{2,p}^{(l)}/(p(p+2)) =
b_{2,p}^{(l)}/120$, which can take any value in $[1, \infty)$
under the alternative hypothesis.
The ``ceiling'' ($\approx 16.5$ in the caption) refers not to
a maximum of $\kl$ itself, but to the maximum value of
$\kl$ for which Mardia's \emph{asymptotic critical region}
is calibrated for this $(n,p)$ combination:
for $n=300$, $p=10$, the statistic $N b_{2,p} / (p(p+2))$
has a reference chi-squared distribution whose upper tail
becomes unreliable beyond $\kl \approx 16.5$.
The observed values $\kl \approx 24$--$33$ lie well above this
calibration range, which means not that $\kl$ is capped but
that the test can only report that $\kl$ is \emph{larger than
the ceiling}, not by how much.
In other words, the ceiling is a limitation of the \emph{test
instrument}, not of the kurtosis itself.
The kurtosis increases from $\kl \approx 24$--$25$ at
initialisation to $\kl \approx 33$ post-training across all
levels, reflecting score polarisation: as the model learns
sharp attention patterns, score distributions develop heavier
tails that push $\kl$ further from the Gaussian reference $\kl=1$.

\begin{table}[htbp]
\centering
\caption{Mardia's normalised kurtosis $\kl = b_{2,p}/(p(p+2))$
($d=128$, $n=300$, $p=10$).
Gaussian reference: $\kl=1$.
The test instrument is calibrated up to $\kl \approx 16.5$
(its asymptotic critical region); values above this threshold
indicate $\kl$ exceeds the ceiling, not that $\kl$ is bounded.}
\label{tab:mardia}
\begin{tabular}{@{}lcccc@{}}
\toprule
& \multicolumn{2}{c}{Pre-training} &
  \multicolumn{2}{c}{Post-training} \\
\cmidrule(lr){2-3}\cmidrule(lr){4-5}
Level & $b_{2,p}$ & $\kl$ & $b_{2,p}$ & $\kl$ \\
\midrule
$l=0$ & 2978.7 & 24.82 & 3980.0 & 33.17 \\
$l=1$ & 2951.9 & 24.60 & 3980.0 & 33.17 \\
$l=2$ & 2793.6 & 23.28 & 3980.0 & 33.17 \\
\bottomrule
\end{tabular}
\end{table}

\paragraph{Width scaling.}
At $n \leq 5{,}000$, $p \leq 2$ (ceiling $\leq 1{,}250$),
the statistic saturates the ceiling $b_{2,p}^{\max} = n \cdot p$
for all widths $d \in \{128, 256, 512\}$.
Empirical verification of the $\mathcal{O}(1/d)$ scaling
of \citet{roberts2022} is therefore not possible within the
tested range; whether H2' becomes accurate for $d \gg 512$
remains open.

\begin{figure}[pos=htbp]
\centering
\includegraphics[width=\linewidth]{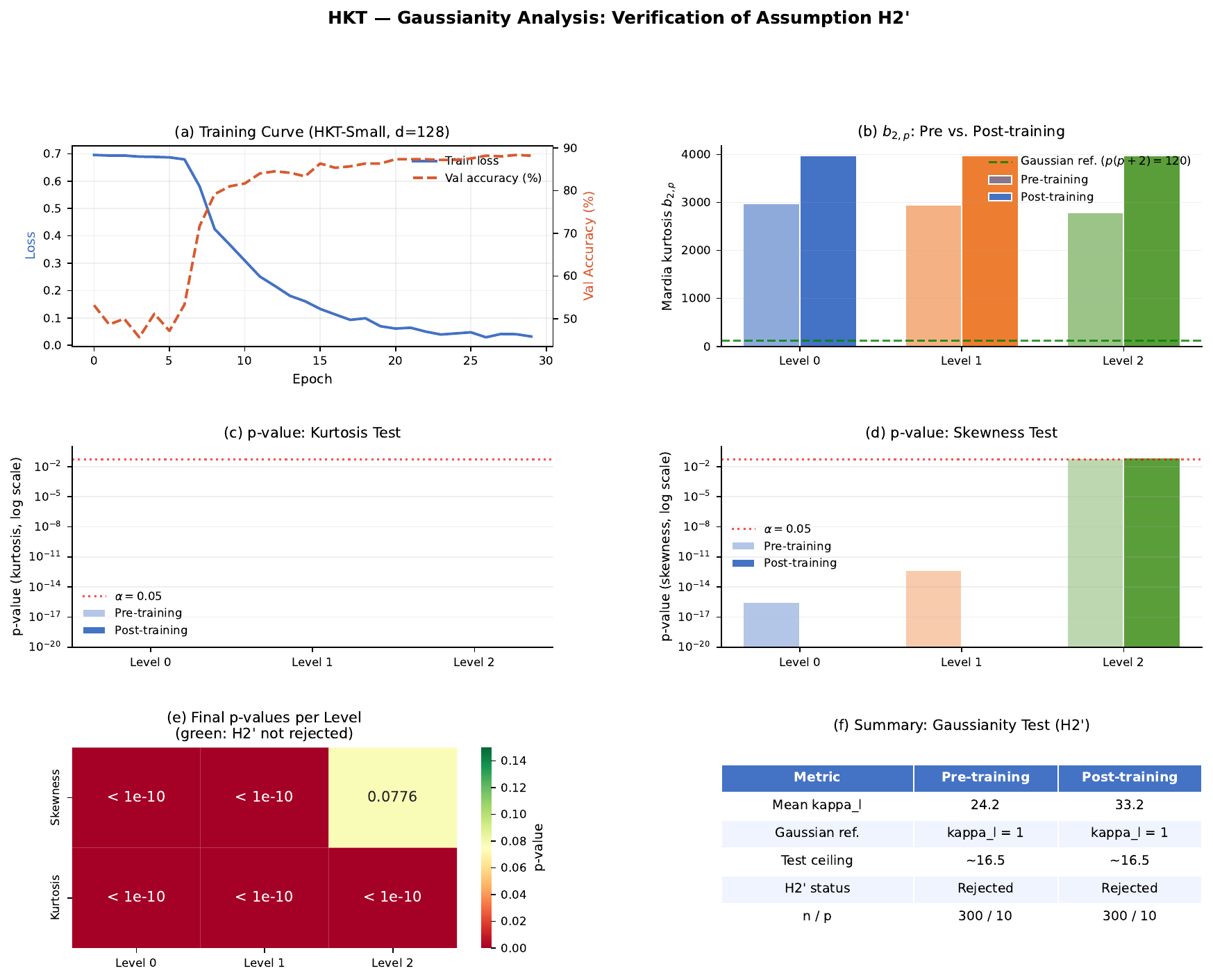}
\caption{Evolution of $\kl$ during training ($d=128$, 30 epochs).
Top: training loss/accuracy and $b_{2,p}$ vs.\ epoch.
Centre: p-values for kurtosis and skewness.
Bottom: QQ-plot of the first principal component of $S^{(0)}$
and heatmap of final p-values.}
\label{fig:gaussianity}
\end{figure}

\section*{Declaration of competing interest}

The author declares no competing interests.

\section*{Data availability}

The synthetic ListOps dataset generator, model code, and
trained checkpoints will be made publicly available upon
acceptance of the manuscript.

\bibliographystyle{cas-model2-names}
\bibliography{hkt_refs}

\end{document}